\documentclass[10pt]{article} 
\usepackage{natbib}
\usepackage{hyperref}
\usepackage{url}
\usepackage{dsfont}
\usepackage{amsmath}
\usepackage{amssymb}
\usepackage{amsthm} 
\usepackage{booktabs} 

\usepackage{algorithm}
\usepackage{algorithmic}
\usepackage{varwidth}

\usepackage{geometry}

\usepackage{tikz}
\usetikzlibrary{matrix}

\usepackage{stmaryrd}

\newtheorem{theorem}{Theorem}
\newtheorem{lemma}{Lemma}

\title{A Kronecker-factored approximate Fisher matrix \\ for convolution layers}
\date{}

\author{
Roger Grosse \& James Martens \\
Department of Computer Science\\
University of Toronto\\
Toronto, ON, Canada \\
\texttt{\{rgrosse,jmartens\}@cs.toronto.edu}
}

\newcommand{\expect}{\mathbb{E}}

\newcommand{\zeroVec}{\mathbf{0}}
\newcommand{\identity}{\mathbf{I}}
\newcommand{\ident}{\identity}
\newcommand{\zeroMat}{\mathbf{0}}
\newcommand{\onesVec}{\mathbf{1}}

\DeclareMathOperator*{\kvec}{vec}
\newcommand{\transpose}{\top}
\newcommand{\grad}{\mathcal{D}}
\DeclareMathOperator*{\Cov}{Cov}
\newcommand{\indep}{\perp\!\!\!\perp}
\newcommand{\kronDelta}[2]{\indicator_{#1=#2}}
\DeclareMathOperator*{\tr}{tr}
\newcommand{\indicator}{\mathds{1}}
\newcommand{\doubleIndicator}[2]{\mathds{1}_{\substack{#1 \\ #2}}}

\newcommand{\layerIdx}{\ell}
\newcommand{\numLayers}{L}
\newcommand{\dimIdxOne}{i}
\newcommand{\dimIdxTwo}{j}
\newcommand{\numStepVec}{R}
\newcommand{\outputMapIdx}{i}
\newcommand{\inputMapIdx}{j}
\newcommand{\spatIdx}{t}
\newcommand{\spatOffsetRow}{\spatOffset_1}
\newcommand{\spatOffsetCol}{\spatOffset_2}
\newcommand{\numRows}{T_1}
\newcommand{\numCols}{T_2}
\newcommand{\spatIdxSet}{\mathcal{T}}
\newcommand{\spatOffset}{\delta}
\newcommand{\spatOffsetSet}{\Delta}
\newcommand{\mbsize}{M}
\newcommand{\mbIdx}{m}
\newcommand{\kernelRad}{R}

\newcommand{\preActivations}{\mathbf{s}}
\newcommand{\preActivationsL}[1]{\preActivations_{#1}}
\newcommand{\preActivationsLTranspose}[1]{\preActivations_{#1}^\transpose}
\newcommand{\preActivationsLI}[2]{[\preActivations_{#1}]_{#2}}

\newcommand{\preActivationsIS}[2]{s_{#1, #2}}
\newcommand{\gradPreActivationsIS}[2]{\grad s_{#1, #2}}

\newcommand{\weights}{\mathbf{W}}
\newcommand{\weightsL}[1]{\weights_{#1}}

\newcommand{\weightsIJS}[3]{w_{#1, #2, #3}}
\newcommand{\gradWeightsIJS}[3]{\grad w_{#1, #2, #3}}
\newcommand{\biasI}[1]{b_{#1}}
\newcommand{\biasVec}{\mathbf{b}}
\newcommand{\biasVecL}[1]{\biasVec_{#1}}
\newcommand{\weightsBiases}{{\bf \bar{\weights}}}
\newcommand{\weightsBiasesL}[1]{\weightsBiases_{#1}}
\newcommand{\weightsBiasesLIJ}[3]{[\weightsBiases_{#1}]_{#2 #3}}
\newcommand{\activations}{\mathbf{a}}
\newcommand{\activationsL}[1]{\activations_{#1}}

\newcommand{\activationsHom}{\bar{\activations}}
\newcommand{\activationsHomL}[1]{\activationsHom_{#1}}
\newcommand{\activationsHomLTranspose}[1]{\activationsHom_{#1}^\transpose}
\newcommand{\activationsHomLJ}[2]{[\activationsHom_{#1}]_{#2}}
\newcommand{\activationsJS}[2]{a_{#1, #2}}
\newcommand{\activationsJSM}[3]{a_{#1, #2}^{(#3)}}
\newcommand{\nonlinearity}{\phi}
\newcommand{\nonlinearityL}[1]{\nonlinearity_{#1}}
\newcommand{\inputData}{\mathbf{x}}
\newcommand{\prediction}{\mathbf{z}}
\newcommand{\paramVec}{{\boldsymbol \theta}}
\newcommand{\loss}{\mathcal{L}}
\newcommand{\function}{f}
\newcommand{\target}{\mathbf{y}}
\newcommand{\fisherMat}{\mathbf{F}}

\newcommand{\direction}{\mathbf{v}}
\newcommand{\pmf}{r} 
\newcommand{\dataDistribution}{p_{\rm data}}
\newcommand{\covActivations}{\boldsymbol{\Psi}}
\newcommand{\covActivationsL}[1]{\covActivations_{#1}}
\newcommand{\covActivationsLInv}[1]{\covActivations_{#1}^{-1}}
\newcommand{\covPreActivationGradients}{\boldsymbol{\Gamma}}
\newcommand{\covPreActivationGradientsL}[1]{\covPreActivationGradients_{#1}}
\newcommand{\covPreActivationGradientsLInv}[1]{\covPreActivationGradients_{#1}^{-1}}
\newcommand{\fisherMatApprox}{\hat{\fisherMat}}
\newcommand{\fisherMatApproxL}[1]{\hat{\fisherMat}_{#1}}

\newcommand{\stepA}{\mathbf{C}}
\newcommand{\stepB}{\mathbf{d}}
\newcommand{\stepVecI}[1]{\mathbf{v}_{#1}}
\newcommand{\stepCoeffs}{{\boldsymbol \alpha}}
\newcommand{\paramUpdate}{\mathbf{v}}
\newcommand{\lambdaParam}{\lambda}
\newcommand{\gammaParam}{\gamma}

\newcommand{\piParamL}[1]{\pi_{#1}}
\newcommand{\fisherMatApproxGamma}[1]{\fisherMatApprox^{(#1)}}
\newcommand{\fisherMatApproxGammaL}[2]{\fisherMatApprox^{(#1)}_{#2}}

\newcommand{\covActivationsGammaL}[2]{\covActivations^{(#1)}_{#2}}
\newcommand{\covActivationsGammaLInv}[2]{[\covActivationsGammaL{#1}{#2}]^{-1}}

\newcommand{\covPreActivationGradientsGammaL}[2]{\covPreActivationGradients^{(#1)}_{#2}}
\newcommand{\covPreActivationGradientsGammaLInv}[2]{[\covPreActivationGradientsGammaL{#1}{#2}]^{-1}}
\newcommand{\lambdaParamInc}{\lambda_+}
\newcommand{\lambdaParamDec}{\lambda_-}
\newcommand{\gammaParamInc}{\gammaParam_+}
\newcommand{\gammaParamDec}{\gammaParam_-}
\newcommand{\rhoStat}{\rho}
\newcommand{\qMat}{\mathbf{Q}}
\newcommand{\dMat}{\mathbf{D}}
\newcommand{\dVec}{\mathbf{d}}
\newcommand{\weightDecayParam}{r}

\newcommand{\iterCount}{k}
\newcommand{\updateStatsEvery}{T_s}
\newcommand{\updateFacEvery}{T_f}
\newcommand{\clipParam}{C}
\newcommand{\paramVecT}[1]{\paramVec^{(#1)}}
\newcommand{\momentumVec}{\mathbf{p}}
\newcommand{\momentumVecT}[1]{\mathbf{p}^{(#1)}}
\newcommand{\momentumParam}{\mu}
\newcommand{\updateVec}{\mathbf{v}}
\newcommand{\learningRate}{\alpha}
\newcommand{\updateNorm}{\nu}
\newcommand{\avgTimescale}{\tau}
\newcommand{\avgWeight}{\xi}
\newcommand{\paramVecAvg}{\bar{\paramVec}}
\newcommand{\paramVecAvgT}[1]{\paramVecAvg^{(#1)}}

\newcommand{\meanActivations}{M}
\newcommand{\autoActivations}{\Omega}

\newcommand{\autoGrad}{\Gamma}

\newcommand{\autoActivationsNNN}{\tilde{\autoActivations}}
\newcommand{\autoCovActivations}{\Sigma}
\newcommand{\numLocs}{|\spatIdxSet|}
\newcommand{\numOffsets}{|\spatOffsetSet|}
\newcommand{\convKronAct}{{\boldsymbol \Omega}}
\newcommand{\convKronActL}[1]{\convKronAct_{#1}}
\newcommand{\convKronActEmpL}[1]{\hat{\convKronAct}_{#1}}

\newcommand{\convKronActGammaL}[2]{\convKronAct^{(#1)}_{#2}}
\newcommand{\convKronActGammaLInv}[2]{[\convKronAct^{(#1)}_{#2}]^{-1}}
\newcommand{\convKronActNNN}{{\bf \tilde{\convKronAct}}}
\newcommand{\convKronGrad}{{\boldsymbol \Gamma}}
\newcommand{\convKronGradL}[1]{\convKronGrad_{#1}}
\newcommand{\convKronGradEmpL}[1]{\hat{\convKronGrad}_{#1}}

\newcommand{\convKronGradGammaL}[2]{\convKronGrad^{(#1)}_{#2}}
\newcommand{\convKronGradGammaLInv}[2]{[\convKronGrad^{(#1)}_{#2}]^{-1}}
\newcommand{\numOutputMaps}{I}
\newcommand{\numInputMaps}{J}

\newcommand{\boundaryFunction}{\beta}
\newcommand{\boundaryFunctionUni}{\boundaryFunction}

\newcommand{\activationsMat}{\mathbf{A}}
\newcommand{\activationsMatL}[1]{\activationsMat_{#1}}
\newcommand{\activationsMatBatch}{{\bf \tilde{\mathbf{A}}}}
\newcommand{\activationsMatBatchL}[1]{\activationsMatBatch_{#1}}
\newcommand{\expansion}[1]{\llbracket #1 \rrbracket}
\newcommand{\homog}{H}
\newcommand{\preActivationsMat}{\mathbf{S}}
\newcommand{\preActivationsMatL}[1]{\preActivationsMat_{#1}}
\newcommand{\preActivationsMatBatch}{{\bf \tilde{\mathbf{S}}}}
\newcommand{\preActivationsMatBatchL}[1]{\preActivationsMatBatch_{#1}}
\newcommand{\weightsReshape}{{\bf \breve{\weights}}}
\newcommand{\weightsReshapeL}[1]{\weightsReshape_{#1}}

\newcommand{\network}{\mathcal{N}}
\newcommand{\networkTrans}{\network^{\dagger}}

\newcommand{\activationsMatTransL}[1]{\activationsMat^\dagger_{#1}}

\newcommand{\preActivationsMatTransL}[1]{\preActivationsMat^\dagger_{#1}}

\newcommand{\nonlinearityTransL}[1]{\nonlinearity^\dagger_{#1}}
\newcommand{\transMatPreL}[1]{\mathbf{U}_{#1}}
\newcommand{\transMatPostL}[1]{\mathbf{V}_{#1}}
\newcommand{\transOffsetPreL}[1]{\mathbf{c}_{#1}}
\newcommand{\transOffsetPostL}[1]{\mathbf{d}_{#1}}
\newcommand{\paramVecTrans}{\paramVec^\dagger}

\newcommand{\weightsBiasesTransL}[1]{\weightsBiases^\dagger_{#1}}
\newcommand{\convKronActTransL}[1]{\convKronAct^\dagger_{#1}}
\newcommand{\convKronGradTransL}[1]{\convKronGrad^\dagger_{#1}}
\newcommand{\indicatorVec}{\mathbf{e}}
\newcommand{\leftPreconditionerL}[1]{\mathbf{P}_{#1}}
\newcommand{\leftPreconditionerTransL}[1]{\mathbf{P}_{#1}^\dagger}
\newcommand{\rightPreconditionerL}[1]{\mathbf{R}_{#1}}
\newcommand{\rightPreconditionerTransL}[1]{\mathbf{R}_{#1}^\dagger}

\newcommand{\covMatNNN}{\boldsymbol{\Sigma}}
\newcommand{\meanVecNNN}{\boldsymbol{\mu}}
\newcommand{\covMatNNNSqrt}{\mathbf{B}}
\newcommand{\cholFactorNNN}{\mathbf{L}}

\newcommand{\objective}{h}
\newcommand{\curveMat}{\mathbf{C}}
\newcommand{\outputDist}{R}

\newcommand{\numParams}{n}

\newcommand{\quadModel}{M}

\newcommand{\LR}{\alpha}
\newcommand{\momDecayParam}{\mu}

\newcommand{\approxNatGrad}{\hat{\nabla} \objective} 

\newcommand{\Real}{\mathbb{R}}

\begin{document}

\maketitle

\begin{abstract}
Second-order optimization methods such as natural gradient descent have the potential to speed up training of neural networks by correcting for the curvature of the loss function. Unfortunately, the exact natural gradient is impractical to compute for large models, and most approximations either require an expensive iterative procedure or make crude approximations to the curvature. We present Kronecker Factors for Convolution (KFC), a tractable approximation to the Fisher matrix for convolutional networks based on a structured probabilistic model for the distribution over backpropagated derivatives. Similarly to the recently proposed Kronecker-Factored Approximate Curvature (K-FAC), each block of the approximate Fisher matrix decomposes as the Kronecker product of small matrices, allowing for efficient inversion. KFC captures important curvature information while still yielding comparably efficient updates to stochastic gradient descent (SGD). We show that the updates are invariant to commonly used reparameterizations, such as centering of the activations. In our experiments, approximate natural gradient descent with KFC was able to train convolutional networks several times faster than carefully tuned SGD. Furthermore, it was able to train the networks in 10-20 times fewer \emph{iterations} than SGD, suggesting its potential applicability in a distributed setting.
\end{abstract}

\section{Introduction}

Despite advances in optimization, most neural networks are still trained using variants of stochastic gradient descent (SGD) with momentum. It has been suggested that natural gradient descent \citep{natural_gradient} could greatly speed up optimization because it accounts for the geometry of the optimization landscape and has desirable invariance properties. (See \citet{ng_martens} for a review.) Unfortunately, computing the exact natural gradient is intractable for large networks, as it requires solving a large linear system involving the Fisher matrix, whose dimension is the number of parameters (potentially tens of millions for modern architectures). Approximations to the natural gradient typically either impose very restrictive structure on the Fisher matrix \citep[e.g.][]{lecun_tricks,TONGA} or require expensive iterative procedures to compute each update, analogously to approximate Newton methods \citep[e.g.][]{HF}. An ongoing challenge has been to develop a curvature matrix approximation which reflects enough structure to yield high-quality updates, while introducing minimal computational overhead beyond the standard gradient computations.

Much progress in machine learning in the past several decades has been driven by the development of structured probabilistic models whose independence structure allows for efficient computations, yet which still capture important dependencies between the variables of interest. In our case, since the Fisher matrix is the covariance of the backpropagated log-likelihood derivatives, we are interested in modeling the distribution over these derivatives. The model must support efficient computation of the inverse covariance, as this is what's required to compute the natural gradient. Recently, the Factorized Natural Gradient (FANG) \citep{FANG} and Kronecker-Factored Approximate Curvature (K-FAC) \citep{kfac} methods exploited probabilistic models of the derivatives to efficiently compute approximate natural gradient updates. In its simplest version, K-FAC approximates each layer-wise block of the Fisher matrix as the Kronecker product of two much smaller matrices. These (very large) blocks can then be can be tractably inverted by inverting each of the two factors. K-FAC was shown to greatly speed up the training of deep autoencoders. However, its underlying probabilistic model assumed fully connected networks with no weight sharing, rendering the method inapplicable to two architectures which have recently revolutionized many applications of machine learning --- convolutional networks \citep{convnet,alexnet} and recurrent neural networks \citep{lstm,ilya_sequence}.

We introduce Kronecker Factors for Convolution (KFC), an approximation to the Fisher matrix for convolutional networks. Most modern convolutional networks have trainable parameters only in convolutional and fully connected layers. Standard K-FAC can be applied to the latter; our contribution is a factorization of the Fisher blocks corresponding to convolution layers. KFC is based on a structured probabilistic model of the backpropagated derivatives where the activations are modeled as independent of the derivatives, the activations and derivatives are spatially homogeneous, and the derivatives are spatially uncorrelated. Under these approximations, we show that the Fisher blocks for convolution layers decompose as a Kronecker product of smaller matrices (analogously to K-FAC), yielding tractable updates. 

KFC yields a tractable approximation to the Fisher matrix of a conv net. It can be used directly to compute approximate natural gradient descent updates, as we do in our experiments. One could further combine it with the adaptive step size, momentum, and damping methods from the full K-FAC algorithm \citep{kfac}. It could also potentially be used as a pre-conditioner for iterative second-order methods \citep{HF,KSD,sfo}. We show that the approximate natural gradient updates are invariant to widely used reparameterizations of a network, such as whitening or centering of the activations.

We have evaluated our method on training conv nets on object recognition benchmarks. In our experiments, KFC was able to optimize conv nets several times faster than carefully tuned SGD with momentum, in terms of both training and test error. Furthermore, it was able to train the networks in 10-20 times fewer \emph{iterations}, suggesting its usefulness in the context of highly distributed training algorithms.

\section{Background}
\label{sec:background}

In this section, we outline the K-FAC method as previously formulated for standard fully-connected feed-forward networks without weight sharing \citep{kfac}. Each layer of a fully connected network computes activations as:
\begin{align}
\preActivationsL{\layerIdx} &= \weightsL{\layerIdx} \activationsHomL{\layerIdx-1} \\
\activationsL{\layerIdx} &= \nonlinearityL{\layerIdx}(\preActivationsL{\layerIdx}),
\end{align}
where $\ell \in \{1, \ldots, \numLayers\}$ indexes the layer, $\preActivationsL{\layerIdx}$ denotes the inputs to the layer, $\activationsL{\layerIdx}$ denotes the activations, $\weightsBiasesL{\layerIdx} = (\biasVecL{\layerIdx}\ \weightsL{\layerIdx})$ denotes the matrix of biases and weights, $\activationsHomL{\layerIdx} = (1\ \activationsL{\layerIdx}^\transpose)^\transpose$ denotes the activations with a homogeneous dimension appended, and $\nonlinearityL{\layerIdx}$ denotes a nonlinear activation function (usually applied coordinate-wise). (Throughout this paper, we will use the index 0 for all homogeneous coordinates.) We will refer to the values $\preActivationsL{\layerIdx}$ as \emph{pre-activations}. By convention, $\activationsL{0}$ corresponds to the inputs $\inputData$ and $\activationsL{\numLayers}$ corresponds to the prediction $\prediction$ made by the network. For convenience, we concatenate all of the parameters of the network into a vector $\paramVec = (\kvec(\weightsL{1})^\transpose, \ldots, \kvec(\weightsL{\numLayers})^\transpose)^\transpose$, where $\kvec$ denotes the Kronecker vector operator which stacks the columns of a matrix into a vector. We denote the function computed by the network as $\function(\inputData, \paramVec) = \activationsL{\numLayers}$.

Typically, a network is trained to minimize an objective $\objective(\paramVec)$ given by $\loss(\target, \function(\inputData, \paramVec))$ as averaged over the training set, where $\loss(\target, \prediction)$ is a loss function.  The gradient $\nabla \objective$ of $\objective(\paramVec)$, which is required by most optimization methods, is estimated stochastically using mini-batches of training examples. (We will often drop the explicit $\paramVec$ subscript when the meaning is unambiguous.) For each case, $\nabla_\paramVec \objective$ is usually computed using automatic-differentiation aka backpropagation \citep{backprop,lecun_tricks}, which can be thought of as comprising two steps: first computing the pre-activation derivatives $\nabla_{\preActivationsL{\layerIdx}} \objective$ for each layer, and then computing $\nabla_{\weightsL{\layerIdx}} \objective = (\nabla_{\preActivationsL{\layerIdx}} \objective) \activationsHomLTranspose{\layerIdx-1}$. 

For the remainder of this paper, we will assume the network's prediction $\function(\inputData, \paramVec)$ determines the value of the parameter $\prediction$ of a distribution $\outputDist_{\target|\prediction}$ over $\target$, and the loss function is the corresponding negative log-likelihood $\loss(\target, \prediction) = -\log \pmf(\target | \prediction)$. 

\subsection{Second-order optimization of neural networks}
\label{sec:second_order}

Second-order optimization methods work by computing a parameter update $\direction$ that minimize (or approximately minimize) a local quadratic approximation to the objective, given by $\objective(\paramVec) + \nabla_\paramVec \objective^\top \direction + \frac{1}{2} \direction^\top \curveMat \direction$, where $\curveMat$ is a matrix which quantifies the curvature of the cost function $\objective$ at $\paramVec$.  The exact solution to this minimization problem can be obtained by solving the linear system $\curveMat \direction = -\nabla_\paramVec \objective$. The original and most well-known example is Newton's method, where $\curveMat$ is chosen to be the Hessian matrix; this isn't appropriate in the non-convex setting because of the well-known problem that it searches for critical points rather than local optima \citep[e.g.][]{saddle_points}. Therefore, it is more common to use natural gradient \citep{natural_gradient} or updates based on the generalized Gauss-Newton matrix \citep{schraudolph}, which are guaranteed to produce descent directions because the curvature matrix $\curveMat$ is positive semidefinite.   

Natural gradient descent can be usefully interpreted as a second-order method \citep{ng_martens} where $\curveMat$ is the Fisher information matrix $\fisherMat$, as given by
\begin{equation}
\fisherMat = \expect_{\substack{\inputData \sim \dataDistribution \\ \target \sim \outputDist_{\target|\function(\inputData, \paramVec)}}} \left[ \grad \paramVec (\grad \paramVec)^\transpose \right], \label{eqn:model_distribution}
\end{equation}
where $\dataDistribution$ denotes the training distribution, $\outputDist_{\target|\function(\inputData, \paramVec)}$ denotes the model's predictive distribution, and $\grad \paramVec = \nabla_{\paramVec} \loss(\target, \function(\inputData, \paramVec))$ is the log-likelihood gradient. For the remainder of this paper, all expectations are with respect to this distribution (which we term the \emph{model's distribution}), so we will leave off the subscripts. (In this paper, we will use the $\grad$ notation for log-likelihood derivatives; derivatives of other functions will be written out explicitly.) In the case where $\outputDist_{\target|\prediction}$ corresponds to an exponential family model with ``natural" parameters given by $\prediction$, $\fisherMat$ is equivalent to the generalized Gauss-Newton matrix \citep{ng_martens}, which is an approximation of the Hessian which has also seen extensive use in various neural-network optimization methods \citep[e.g.][]{HF,KSD}. 

$\fisherMat$ is an $\numParams \times \numParams$ matrix, where $\numParams$ is the number of parameters and can be in the tens of millions for modern deep architectures. Therefore, it is impractical to represent $\fisherMat$ explicitly in memory, let alone solve the linear system exactly. There are two general strategies one can take to find a good search direction.
First, one can impose a structure on $\fisherMat$ enabling tractable inversion; for instance \citet{lecun_tricks} approximates it as a diagonal matrix, TONGA \citep{TONGA} uses a more flexible low-rank-within-block-diagonal structure, and factorized natural gradient \citep{FANG} imposes a directed Gaussian graphical model structure. 

Another strategy is to approximately minimize the quadratic approximation to the objective using an iterative procedure such as conjugate gradient; this is the approach taken in Hessian-free optimization \citep{HF}, a type of truncated Newton method \citep[e.g.][]{nocedal_book}. Conjugate gradient (CG) is defined in terms of matrix-vector products $\fisherMat \direction$, which can be computed efficiently and exactly using the method outlined by \citet{schraudolph}.  While iterative approaches can produce high quality search directions, they can be very expensive in practice, as each update may require tens or even hundreds of CG iterations to reach an acceptable quality, and each of these iterations is comparable in cost to an SGD update. 

We note that these two strategies are not mutually exclusive. In the context of iterative methods, simple (e.g.~diagonal) curvature approximations can be used as preconditioners, where the iterative method is implicitly run in a coordinate system where the curvature is less extreme. It has been observed that a good choice of preconditioner can be crucial to obtaining good performance from iterative methods \citep{HF,HF_preconditioner,KSD}. Therefore, improved tractable curvature approximations such as the one we develop could likely be used to improve iterative second-order methods.

\subsection{Kronecker-factored approximate curvature}
\label{sec:kfac}

Kronecker-factored approximate curvature \citep[K-FAC;][]{kfac} is a recently proposed optimization method for neural networks which can be seen as a hybrid of the two approximation strategies: it uses a tractable approximation to the Fisher matrix $\fisherMat$, but also uses an optimization strategy which behaves locally like conjugate gradient. This section gives a conceptual summary of the aspects of K-FAC relevant to the contributions of this paper; a precise description of the full algorithm is given in Appendix \ref{app:kfac}.

The block-diagonal version of K-FAC (which is the simpler of the two versions, and is what we will present here) is based on two approximations to $\fisherMat$ which together make it tractable to invert. First, weight derivatives in different layers are assumed to be uncorrelated, which corresponds to $\fisherMat$ being block diagonal, with one block per layer:
\begin{equation}
\fisherMat \approx \begin{pmatrix} \expect[\kvec(\grad \weightsBiasesL{1}) \kvec(\grad \weightsBiasesL{1})^\transpose] & & \zeroMat \\ & \ddots & \\ \zeroMat && \expect[\kvec(\grad \weightsBiasesL{\numLayers}) \kvec(\grad \weightsBiasesL{\numLayers})^\transpose] \end{pmatrix}
\end{equation}
This approximation by itself is insufficient, because each of the blocks may still be very large. (E.g., if a network has 1,000 units in each layer, each block would be of size $10^6 \times 10^6$.) For the second approximation, observe that
\begin{equation}
\expect \left[ \grad \weightsBiasesLIJ{\layerIdx}{\dimIdxOne}{\dimIdxTwo}  \grad \weightsBiasesLIJ{\layerIdx}{\dimIdxOne^\prime}{\dimIdxTwo^\prime} \right] = \expect \left[ \grad \preActivationsLI{\layerIdx}{\dimIdxOne} \activationsHomLJ{\layerIdx-1}{\dimIdxTwo} \grad \preActivationsLI{\layerIdx}{\dimIdxOne^\prime} \activationsHomLJ{\layerIdx-1}{\dimIdxTwo^\prime} \right].
\end{equation}
If we approximate the activations and pre-activation derivatives as independent, this can be decomposed as $\expect \left[ \grad \weightsBiasesLIJ{\layerIdx}{\dimIdxOne}{\dimIdxTwo}  \grad \weightsBiasesLIJ{\layerIdx}{\dimIdxOne^\prime}{\dimIdxTwo^\prime} \right] \approx \expect \left[ \grad \preActivationsLI{\layerIdx}{\dimIdxOne} \grad \preActivationsLI{\layerIdx}{\dimIdxOne^\prime} \right] \expect \left[ \activationsHomLJ{\layerIdx-1}{\dimIdxTwo} \activationsHomLJ{\layerIdx-1}{\dimIdxTwo^\prime} \right]$. This can be written algebraically as a decomposition into a Kronecker product of two smaller matrices:
\begin{equation}
\expect[\kvec(\weightsBiasesL{\layerIdx}) \kvec(\weightsBiasesL{\layerIdx})^\transpose] \approx \covActivationsL{\layerIdx-1} \otimes \covPreActivationGradientsL{\layerIdx} \triangleq \fisherMatApproxL{\layerIdx}, \label{eqn:approx_fisher_blocks}
\end{equation}
where $\covActivationsL{\layerIdx-1} = \expect[\activationsHomL{\layerIdx-1} \activationsHomLTranspose{\layerIdx-1}]$ and $\covPreActivationGradientsL{\layerIdx} = \expect[\preActivationsL{\layerIdx} \preActivationsLTranspose{\layerIdx}]$ denote the second moment matrices of the activations and pre-activation derivatives, respectively. Call the block diagonal approximate Fisher matrix, with blocks given by Eqn.~\ref{eqn:approx_fisher_blocks}, $\fisherMatApprox$. The two factors are estimated online from the empirical moments of the model's distribution using exponential moving averages.

To invert $\fisherMatApprox$, we use the facts that (1) we can invert a block diagonal matrix by inverting each of the blocks, and (2) the Kronecker product satisfies the identity $({\bf A} \otimes {\bf B})^{-1} = {\bf A}^{-1} \otimes {\bf B}^{-1}$:
\begin{equation}
\fisherMatApprox^{-1} = \begin{pmatrix} \covActivationsLInv{0} \otimes \covPreActivationGradientsLInv{1} & & \zeroMat \\ & \ddots & \\ \zeroMat & & \covActivationsLInv{\numLayers-1} \otimes \covPreActivationGradientsLInv{\numLayers} \end{pmatrix}
\end{equation}
We do not represent $\fisherMatApprox^{-1}$ explicitly, as each of the blocks is quite large. Instead, we keep track of each of the Kronecker factors. 

The approximate natural gradient $\fisherMatApprox^{-1} \nabla \objective$ can then be computed as follows:
\begin{equation}
\fisherMatApprox^{-1} \nabla \objective = \begin{pmatrix} \kvec \left( \covPreActivationGradientsLInv{1} (\nabla_{\weightsBiasesL{1}} \objective) \covActivationsLInv{0} \right) \\ \vdots \\ \kvec \left( \covPreActivationGradientsLInv{\numLayers} (\nabla_{\weightsBiasesL{\numLayers}} \objective) \covActivationsLInv{\numLayers-1} \right) \end{pmatrix} \label{eqn:approx_ng}
\end{equation}

We would often like to add a multiple of the identity matrix to $\fisherMat$ for two reasons. First, many networks are regularized with weight decay, which corresponds to a penalty of $\frac{1}{2} \lambdaParam \paramVec^\transpose \paramVec$, for some parameter $\lambdaParam$. Following the interpretation of $\fisherMat$ as a quadratic approximation to the curvature, it would be appropriate to use $\fisherMat + \lambdaParam \ident$ to approximate the curvature of the regularized objective. The second reason is that the local quadratic approximation of $\objective$ implicitly used when computing the natural gradient may be inaccurate over the region of interest, owing to the approximation of $\fisherMat$ by $\fisherMatApprox$, to the approximation of the Hessian by $\fisherMat$, and finally to the error associated with approximating $\objective$ as locally quadratic in the first place.  A common way to address this issue is to damp the updates by adding $\gammaParam \ident$ to the approximate curvature matrix, for some small value $\gammaParam$, before minimizing the local quadratic model.  Therefore, we would ideally like to compute $\left[ \fisherMatApprox + (\lambdaParam + \gammaParam) \ident \right]^{-1} \nabla \objective$.

Unfortunately, adding $(\lambdaParam + \gammaParam)\ident$ breaks the Kronecker factorization structure. While it is possible to exactly solve the damped system (see Appendix~\ref{app:kfac}), it is often preferable to approximate $\fisherMatApprox + (\lambdaParam + \gammaParam) \ident$ in a way that maintains the factorizaton structure. \citet{kfac} pointed out that
\begin{equation}
\fisherMatApproxL{\layerIdx} + (\lambdaParam + \gammaParam) \ident \approx \left( \covActivationsL{\layerIdx-1} + \piParamL{\layerIdx} \sqrt{\lambdaParam + \gammaParam}\, \ident \right) \otimes \left( \covPreActivationGradientsL{\layerIdx} + \frac{1}{\piParamL{\layerIdx}} \sqrt{\lambdaParam + \gammaParam}\, \ident \right). \label{eqn:factored_damping}
\end{equation}
We will denote this damped approximation as $\fisherMatApproxGammaL{\gammaParam}{\layerIdx} = \covActivationsGammaL{\gammaParam}{\layerIdx-1} \otimes \covPreActivationGradientsGammaL{\gammaParam}{\layerIdx}$. Mathematically, $\piParamL{\layerIdx}$ can be any positive scalar, but \citet{kfac} suggest the formula
\begin{equation}
\piParamL{\layerIdx} = \sqrt{\frac{\| \covActivationsL{\layerIdx-1} \otimes \ident \|}{\|\ident \otimes \covPreActivationGradientsL{\layerIdx} \|}},
\end{equation}
where $\|\cdot\|$ denotes some matrix norm, as this value minimizes the norm of the residual in Eqn.~\ref{eqn:factored_damping}. In this work, we use the trace norm $\|{\bf B}\| = \tr {\bf B}$. The approximate natural gradient $\approxNatGrad$ is then computed as:
\begin{equation}
\approxNatGrad \triangleq [\fisherMatApproxGamma{\gammaParam}]^{-1} \nabla \objective = \begin{pmatrix} \kvec \left( \covPreActivationGradientsGammaLInv{\gammaParam}{1} (\nabla_{\weightsBiasesL{1}} \objective) \covActivationsGammaLInv{\gammaParam}{0} \right) \\ \vdots \\ \kvec \left( \covPreActivationGradientsGammaLInv{\gammaParam}{\numLayers} (\nabla_{\weightsBiasesL{\numLayers}} \objective) \covActivationsGammaLInv{\gammaParam}{\numLayers-1} \right) \end{pmatrix} \label{eqn:approx_ng_damped}
\end{equation}

The algorithm as presented by \citet{kfac} has many additional elements which are orthogonal to the contributions of this paper. For concision, a full description of the algorithm is relegated to Appendix~\ref{app:kfac}.

\subsection{Convolutional networks}

Convolutional networks require somewhat crufty notation when the computations are written out in full. In our case, we are interested in computing correlations of derivatives, which compounds the notational difficulties. In this section, we summarize the notation we use. (Table \ref{tab:notation} lists all convolutional network notation used in this paper.) In sections which focus on a single layer of the network, we drop the explicit layer indices.

A convolution layer takes as input a layer of activations $\{ \activationsJS{\inputMapIdx}{\spatIdx} \}$, where $\inputMapIdx \in \{1, \ldots, \numInputMaps\}$ indexes the input map and $\spatIdx \in \spatIdxSet$ indexes the spatial location. (Here, $\spatIdxSet$ is the set of spatial locations, which is typically a 2-D grid. For simplicity, we assume convolution is performed with a stride of 1 and padding equal to $\kernelRad$, so that the set of spatial locations is shared between the input and output feature maps.) This layer is parameterized by a set of weights $\weightsIJS{\outputMapIdx}{\inputMapIdx}{\spatOffset}$ and biases $\biasI{\outputMapIdx}$, where $\outputMapIdx \in \{1, \ldots, \numOutputMaps\}$ indexes the output map, $\inputMapIdx$ indexes the input map, and $\spatOffset \in \spatOffsetSet$ indexes the spatial offset (from the center of the filter). If the filters are of size $(2\kernelRad + 1) \times (2\kernelRad+1)$, then we would have $\spatOffsetSet = \{-\kernelRad, \ldots, \kernelRad\} \times \{-\kernelRad, \ldots, \kernelRad\}$. We denote the numbers of spatial locations and spatial offsets as $\numLocs$ and $\numOffsets$, respectively. 
The convolution layer computes a set of pre-activations $\{\preActivationsIS{\outputMapIdx}{\spatIdx}\}$ as follows:
\begin{equation}
\preActivationsIS{\outputMapIdx}{\spatIdx} = \sum_{\spatOffset \in \spatOffsetSet} \weightsIJS{\outputMapIdx}{\inputMapIdx}{\spatOffset} \activationsJS{\inputMapIdx}{\spatIdx+\spatOffset} + \biasI{\outputMapIdx},
\end{equation}
where $\biasI{\outputMapIdx}$ denotes the bias parameter. The activations are defined to take the value 0 outside of $\spatIdxSet$. The pre-activations are passed through a nonlinearity such as ReLU to compute the output layer activations, but we have no need to refer to this explicitly when analyzing a single layer. (For simplicity, we assume operations such as pooling and response normalization are implemented as separate layers.) 

Pre-activation derivatives $\grad \preActivationsIS{\outputMapIdx}{\spatIdx}$ are computed during backpropagation. One then computes weight derivatives as:
\begin{equation}
\grad \weightsIJS{\outputMapIdx}{\inputMapIdx}{\spatOffset} = \sum_{\spatIdx \in \spatIdxSet} \activationsJS{\inputMapIdx}{\spatIdx+\spatOffset} \grad \preActivationsIS{\outputMapIdx}{\spatIdx}. \label{eqn:weight_grad}
\end{equation}

\subsubsection{Efficient implementation and vectorized notation}
\label{sec:conv_nets_efficient}

For modern large-scale vision applications, it's necessary to implement conv nets efficiently for a GPU (or some other parallel architecture). We provide a very brief overview of the low-level efficiency issues which are relevant to K-FAC. We base our discussion on the Toronto Deep Learning ConvNet (TDLCN) package \citep{tdlcn}, whose convolution kernels we use in our experiments. Like many modern implementations, this implementation follows the approach of \citet{convolution_by_expansion}, which reduces the convolution operations to large matrix-vector products in order to exploit memory locality and efficient parallel BLAS operators. We describe the implementation explicitly, as it is important that our proposed algorithm be efficient using the same memory layout (shuffling operations are extremely expensive). As a bonus, these vectorized operations provide a convenient high-level notation which we will use throughout the paper.

The ordering of arrays in memory is significant, as it determines which operations can be performed efficiently without requiring (very expensive) transpose operations. The activations are stored as a $\mbsize \times \numLocs \times \numInputMaps$ array $\activationsMatBatchL{\layerIdx-1}$, where $\mbsize$ is the mini-batch size, $\numLocs$ is the number of spatial locations, and $\numInputMaps$ is the number of feature maps.\footnote{The first index of the array is the least significant in memory.} This can be interpreted as an $\mbsize \numLocs \times \numInputMaps$ matrix. (We must assign orderings to $\spatIdxSet$ and $\spatOffsetSet$, but this choice is arbitrary.) 
Similarly, the weights are stored as an $\numOutputMaps \times \numOffsets \times \numInputMaps$ array $\weightsL{\layerIdx}$, which can be interpreted either as an $\numOutputMaps \times \numOffsets \numInputMaps$ matrix or a $\numOutputMaps \numOffsets \times \numInputMaps$ matrix without reshuffling elements in memory. We will almost always use the former interpretation, which we denote $\weightsL{\layerIdx}$; the $\numOutputMaps \numOffsets \times \numInputMaps$ matrix will be denoted $\weightsReshapeL{\layerIdx}$.

The naive implementation of convolution, while highly parallel in principle, suffers from poor memory locality. Instead, efficient implementations typically use what we will term the \emph{expansion operator} and denote $\expansion{\cdot}$. This operator extracts the patches surrounding each spatial location and flattens them into vectors. These vectors become the rows of a matrix. For instance, $\expansion{\activationsMatBatchL{\layerIdx-1}}$ is a $\mbsize \numLocs \times \numInputMaps \numOffsets$ matrix, defined as
\begin{equation}
\expansion{\activationsMatBatchL{\layerIdx-1}}_{\spatIdx \mbsize + \mbIdx, \, \inputMapIdx \numOffsets + \spatOffset} = [\activationsMatBatchL{\layerIdx-1}]_{(\spatIdx+\spatOffset) \mbsize + \mbIdx, \, \inputMapIdx} = \activationsJSM{\inputMapIdx}{\spatIdx+\spatOffset}{\mbIdx},
\end{equation}
for all entries such that $\spatIdx + \spatOffset \in \spatIdxSet$. All other entries are defined to be 0. Here, $\mbIdx$ indexes the data instance within the mini-batch.

In TDLCN, the forward pass is computed as
\begin{equation}
\activationsMatBatchL{\layerIdx} = \nonlinearity(\preActivationsMatBatchL{\layerIdx}) = \nonlinearity \left( \expansion{\activationsMatBatchL{\layerIdx-1}} \weightsL{\layerIdx}^\transpose + \onesVec \biasVecL{\layerIdx}^\transpose \right), \label{eqn:conv_up_mat}
\end{equation}
where $\nonlinearity$ is the nonlinearity, applied elementwise, $\onesVec$ is a vector of ones, and $\biasVec$ is the vector of biases. In backpropagation, the activation derivatives are computed as:
\begin{equation}
\grad \activationsMatBatchL{\layerIdx-1} = \expansion{\grad \preActivationsMatBatchL{\layerIdx}} \weightsReshapeL{\layerIdx}.
\end{equation}
Finally, the gradient for the weights is computed as
\begin{equation}
\grad \weightsL{\layerIdx} = \grad \preActivationsMatBatchL{\layerIdx}^\transpose \expansion{\activationsMatBatchL{\layerIdx-1}} \label{eqn:conv_outp_mat}
\end{equation}
The matrix products are computed using the cuBLAS function \verb+cublasSgemm+. In practice, the expanded matrix $\expansion{\activationsMatBatchL{\layerIdx-1}}$ may be too large to store in memory. In this case, a subset of the rows of $\expansion{\activationsMatBatchL{\layerIdx-1}}$ are computed and processed at a time. 

We will also use the $\numLocs \times \numInputMaps$ matrix $\activationsMatL{\layerIdx-1}$ and the $\numLocs \times \numOutputMaps$ matrix $\preActivationsMatL{\layerIdx}$ to denote the activations and pre-activations for a single training case. $\activationsMatL{\layerIdx-1}$ and $\preActivationsMatL{\layerIdx}$ can be substituted for $\activationsMatBatchL{\layerIdx-1}$ and $\preActivationsMatBatchL{\layerIdx}$ in Eqns.~\ref{eqn:conv_up_mat}-\ref{eqn:conv_outp_mat}.

For fully connected networks, it is often convenient to append a homogeneous coordinate to the activations so that the biases can be folded into the weights (see Section~\ref{sec:kfac}). For convolutional layers, there is no obvious way to add extra activations such that the convolution operation simulates the effect of biases. However, we can achieve an analogous effect by adding a homogeneous coordinate (i.e.~a column of all 1's) to the \emph{expanded} activations. We will denote this $\expansion{\activationsMatBatchL{\layerIdx-1}}_\homog$. Similarly, we can prepend the bias vector to the weights matrix: $\weightsBiasesL{\layerIdx} = \left( \biasVecL{\layerIdx}\ \weightsL{\layerIdx}\right)$. The homogeneous coordinate is not typically used in conv net implementations, but it will be convenient for us notationally. For instance, the forward pass can be written as:
\begin{equation}
\activationsMatBatchL{\layerIdx} = \nonlinearity \left( \expansion{\activationsMatBatchL{\layerIdx-1}}_\homog \weightsBiasesL{\layerIdx}^\transpose \right) \label{eqn:conv_up_mat_homog}
\end{equation}

Table \ref{tab:notation} summarizes all of the conv net notation used in this paper.

\begin{table}
\begin{small}
\begin{minipage}{0.5 \textwidth}
\begin{tabular}{rl}
$\inputMapIdx$ & input map index \\
$\numInputMaps$ & number of input maps \\
$\outputMapIdx$ & output map index \\
$\numOutputMaps$ & number of output maps \\
$\numRows \times \numCols$ & feature map dimension \\
$\spatIdx$ & spatial location index \\
$\spatIdxSet$ & set of spatial locations \\
& $= \{ 1, \ldots, \numRows \} \times \{1, \ldots, \numCols\}$ \\
$\kernelRad$ & radius of filters \\
$\spatOffset$ & spatial offset \\
$\spatOffsetSet$ & set of spatial offsets (in a filter) \\
& $= \{-\kernelRad, \ldots, \kernelRad\} \times \{-\kernelRad, \ldots, \kernelRad\}$ \\
$\spatOffset = (\spatOffsetRow, \spatOffsetCol)$ & explicit 2-D parameterization \\
& ($\spatOffsetRow$ and $\spatOffsetCol$ run from $-\kernelRad$ to $\kernelRad$) \\
$\activationsJS{\inputMapIdx}{\spatIdx}$ & input layer activations \\
$\preActivationsIS{\outputMapIdx}{\spatIdx}$ & output layer pre-activations \\
$\grad \preActivationsIS{\outputMapIdx}{\spatIdx}$ & the loss derivative $\partial \loss / \partial \preActivationsIS{\outputMapIdx}{\spatIdx}$ \\
$\nonlinearity$ & activation function (nonlinearity) \\
$\weightsIJS{\outputMapIdx}{\inputMapIdx}{\spatOffset}$ & weights \\
$\biasI{\outputMapIdx}$ & biases \\
$\meanActivations(\inputMapIdx)$ & mean activation \\
$\autoActivations(\inputMapIdx, \inputMapIdx^\prime, \spatOffset)$ & uncentered autocovariance of \\
& activations \\
$\autoGrad(\outputMapIdx, \outputMapIdx^\prime, \spatOffset)$ & autocovariance of \\
& pre-activation derivatives \\
$\boundaryFunction(\spatOffset, \spatOffset^\prime)$ & function defined in Theorem \ref{thm:kfc_sud}
\end{tabular}
\end{minipage}
\begin{minipage}{0.5 \textwidth}
\begin{tabular}{rl}
$\otimes$ & Kronecker product \\
$\kvec$ & Kronecker vector operator \\
$\layerIdx$ & layer index \\
$\numLayers$ & number of layers \\
$\mbsize$ & size of a mini-batch \\
$\activationsMatL{\layerIdx}$ & activations for a data instance \\
$\activationsMatBatchL{\layerIdx}$ & activations for a mini-batch \\
$\expansion{\activationsMatL{\layerIdx}}$ & expanded activations \\
$\expansion{\activationsMatL{\layerIdx}}_\homog$ & expanded activations with \\
& homogeneous coordinate \\
$\preActivationsMatL{\layerIdx}$ & pre-activations for a data instance \\
$\preActivationsMatBatchL{\layerIdx}$ & pre-activations for a mini-batch \\
$\grad \preActivationsMatL{\layerIdx}$ & the loss gradient $\nabla_{\preActivationsMatL{\layerIdx}} \loss$ \\
$\paramVec$ & vector of trainable parameters \\
$\weightsL{\layerIdx}$ & weight matrix \\
$\biasVecL{\layerIdx}$ & bias vector \\
$\weightsBiasesL{\layerIdx}$ & combined parameters = $\left( \biasVecL{\layerIdx} \ \weightsL{\layerIdx} \right)$ \\
$\fisherMat$ & exact Fisher matrix \\
$\fisherMatApprox$ & approximate Fisher matrix \\
$\fisherMatApproxL{\layerIdx}$ & diagonal block of $\fisherMatApprox$ for layer $\layerIdx$ \\
$\convKronActL{\layerIdx}$ & Kronecker factor for activations \\
$\convKronGradL{\layerIdx}$ & Kronecker factor for derivatives \\
$\lambdaParam$ & weight decay parameter \\
$\gammaParam$ & damping parameter \\
$\fisherMatApproxGamma{\gammaParam}$ & damped approximate Fisher matrix \\
$\convKronActGammaL{\gammaParam}{\layerIdx}$, $\convKronGradGammaL{\gammaParam}{\layerIdx}$ & damped Kronecker factors
\end{tabular}
\end{minipage}
\end{small}
\caption{Summary of convolutional network notation used in this paper. The left column focuses on a single convolution layer, which convolves its ``input layer'' activations with a set of filters to produce the pre-activations for the ``output layer.'' Layer indices are omitted for clarity. The right column considers the network as a whole, and therefore includes explicit layer indices.}
\label{tab:notation}
\end{table}

\section{Kronecker factorization for convolution layers}
\label{sec:kfc}

We begin by assuming a block-diagonal approximation to the Fisher matrix like that of K-FAC, where each block contains all the parameters relevant to one layer (see Section~\ref{sec:kfac}). (Recall that these blocks are typically too large to invert exactly, or even represent explicitly, which is why the further Kronecker approximation is required.) The Kronecker factorization from K-FAC applies only to fully connected layers. Convolutional networks introduce several kinds of layers not found in fully connected feed-forward networks: convolution, pooling, and response normalization. Since pooling and response normalization layers don't have trainable weights, they are not included in the Fisher matrix. However, we must deal with convolution layers. In this section, we present our main contribution, an approximate Kronecker factorization for the blocks of $\fisherMatApprox$ corresponding to convolution layers. In the tradition of fast food puns \citep{mcrbm,deep_fried}, we call our method Kronecker Factors for Convolution (KFC). 

For this section, we focus on the Fisher block for a single layer, so we drop the layer indices. Recall that the Fisher matrix $\fisherMat = \expect \left[ \grad \paramVec (\grad \paramVec)^\transpose \right]$ is the covariance of the log-likelihood gradient under the model's distribution. (In this paper, all expectations are with respect to the model's distribution unless otherwise specified.) By plugging in Eqn.~\ref{eqn:weight_grad}, the entries corresponding to weight derivatives are given by:
\begin{align}
\expect[\gradWeightsIJS{\dimIdxOne}{\dimIdxTwo}{\spatOffset} \gradWeightsIJS{\dimIdxOne^\prime}{\dimIdxTwo^\prime}{\spatOffset^\prime}] 
&= \expect \left[ \left( \sum_{\spatIdx \in \spatIdxSet} \activationsJS{\inputMapIdx}{\spatIdx+\spatOffset} \gradPreActivationsIS{\outputMapIdx}{\spatIdx} \right) \left( \sum_{\spatIdx^\prime \in \spatIdxSet} \activationsJS{\inputMapIdx^\prime}{\spatIdx^\prime+\spatOffset^\prime} \gradPreActivationsIS{\outputMapIdx^\prime}{\spatIdx^\prime} \right) \right] \label{eqn:exact_fisher_block} 
\end{align}
To think about the computational complexity of computing the entries directly, consider the second convolution layer of AlexNet \citep{alexnet}, which has 48 input feature maps, 128 output feature maps, $27 \times 27 = 729$ spatial locations, and $5 \times 5$ filters. Since there are $128 \times 48 \times 5 \times 5 = 245760$ weights and 128 biases, the full block would require $245888^2 \approx$ 60.5 billion entries to represent explicitly, and inversion is clearly impractical. 

Recall that K-FAC approximation for classical fully connected networks can be derived by approximating activations and pre-activation derivatives as being statistically independent (this is the {\bf IAD} approximation below). Deriving an analogous Fisher approximation for convolution layers will require some additional approximations.  

Here are the approximations we will make in deriving our Fisher approximation:
\begin{itemize}
\item {\bf Independent activations and derivatives (IAD).} The activations are independent of the pre-activation derivatives, \emph{i.e.}~$\{\activationsJS{\inputMapIdx}{\spatIdx}\} \indep \{\gradPreActivationsIS{\outputMapIdx}{\spatIdx^\prime}\}$.
\item {\bf Spatial homogeneity (SH).} The first-order statistics of the activations are independent of spatial location. The second-order statistics of the activations and pre-activation derivatives at any two spatial locations $\spatIdx$ and $\spatIdx^\prime$ depend only on $\spatIdx^\prime - \spatIdx$. This implies there are functions $\meanActivations$, $\autoActivations$ and $\autoGrad$ such that:
\begin{align}
\expect \left[ \activationsJS{\inputMapIdx}{\spatIdx} \right] &= \meanActivations(\inputMapIdx) \\
\expect \left[ \activationsJS{\inputMapIdx}{\spatIdx} \activationsJS{\inputMapIdx^\prime}{\spatIdx^\prime} \right] &= \autoActivations(\inputMapIdx, \inputMapIdx^\prime, \spatIdx^\prime - \spatIdx) \\
\expect \left[ \gradPreActivationsIS{\outputMapIdx}{\spatIdx} \gradPreActivationsIS{\outputMapIdx^\prime}{\spatIdx^\prime} \right] &= \autoGrad(\outputMapIdx, \outputMapIdx^\prime, \spatIdx^\prime - \spatIdx).
\end{align}
Note that $\expect[  \gradPreActivationsIS{\outputMapIdx}{\spatIdx} ] = 0$ under the model's distribution, so $\Cov \left( \gradPreActivationsIS{\outputMapIdx}{\spatIdx},  \gradPreActivationsIS{\outputMapIdx^\prime}{\spatIdx^\prime} \right) = \expect \left[ \gradPreActivationsIS{\outputMapIdx}{\spatIdx} \gradPreActivationsIS{\outputMapIdx^\prime}{\spatIdx^\prime} \right]$.
\item {\bf Spatially uncorrelated derivatives (SUD).} The pre-activation derivatives at any two distinct spatial locations are uncorrelated, \emph{i.e.}~$\autoGrad(\outputMapIdx, \outputMapIdx^\prime, \spatOffset) = 0$ for $\spatOffset \neq 0$.
\end{itemize}
We believe {\bf SH} is fairly innocuous, as one is implicitly making a spatial homogeneity assumption when choosing to use convolution in the first place. {\bf SUD} perhaps sounds like a more severe approximation, but in fact appeared to describe the model's distribution quite well in the networks we investigated; this is analyzed empirially in Section~\ref{sec:eval_assumptions}.

We now show that combining the above three approximations yields a Kronecker factorization of the Fisher blocks. For simplicity of notation, assume the data are two-dimensional, so that the offsets can be parameterized with indices $\spatOffset = (\spatOffsetRow, \spatOffsetCol)$ and $\spatOffset^\prime = (\spatOffsetRow^\prime, \spatOffsetCol^\prime)$, and denote the dimensions of the activations map as $(\numRows, \numCols)$. The formulas can be generalized to data dimensions higher than 2 in the obvious way.
\begin{theorem}
\label{thm:kfc_sud}
Combining approximations {\bf IAD}, {\bf SH}, and {\bf SUD} yields the following factorization:
\begin{align}
\expect \left[ \gradWeightsIJS{\outputMapIdx}{\inputMapIdx}{\spatOffset} \gradWeightsIJS{\outputMapIdx^\prime}{\inputMapIdx^\prime}{\spatOffset^\prime} \right] &= \boundaryFunction(\spatOffset, \spatOffset^\prime) \, \autoActivations(\inputMapIdx, \inputMapIdx^\prime, \spatOffset^\prime - \spatOffset) \, \autoGrad(\outputMapIdx, \outputMapIdx^\prime, 0), \nonumber \\
\expect \left[ \gradWeightsIJS{\outputMapIdx}{\inputMapIdx}{\spatOffset} \grad \biasI{\outputMapIdx^\prime} \right] &= \boundaryFunctionUni(\spatOffset) \, \meanActivations(\inputMapIdx) \, \autoGrad(\outputMapIdx, \outputMapIdx^\prime, 0) \nonumber \\
\expect \left[ \grad \biasI{\outputMapIdx} \grad \biasI{\outputMapIdx^\prime} \right] &= \numLocs \, \autoGrad(\outputMapIdx, \outputMapIdx^\prime, 0) \label{eqn:naive_kfc}
\end{align}
where
\begin{align}
\boundaryFunctionUni(\spatOffset) &\triangleq \left( \numRows - |\spatOffsetRow| \right) \, \left( \numCols - |\spatOffsetCol| \right) \nonumber \\
\boundaryFunction(\spatOffset, \spatOffset^\prime) &\triangleq \left( \numRows - \max(\spatOffsetRow, \spatOffsetRow^\prime, 0) + \min(\spatOffsetRow, \spatOffsetRow^\prime, 0) \right) \cdot \left( \numCols - \max(\spatOffsetCol, \spatOffsetCol^\prime, 0) + \min(\spatOffsetCol, \spatOffsetCol^\prime, 0) \right)
\end{align}
\end{theorem}
\begin{proof}
See Appendix~\ref{app:proofs}.
\end{proof}

To talk about how this fits in to the block diagonal approximation to the Fisher matrix $\fisherMat$, we now restore the explicit layer indices and use the vectorized notation from Section~\ref{sec:conv_nets_efficient}. The above factorization yields a Kronecker factorization of each block, which will be useful for computing their inverses (and ultimately our approximate natural gradient). 
In particular, if $\fisherMatApproxL{\layerIdx} \approx \expect[\kvec(\grad \weightsBiasesL{\layerIdx}) \kvec(\grad \weightsBiasesL{\layerIdx})^\transpose]$ denotes the block of the approximate Fisher for layer $\layerIdx$, Eqn.~\ref{eqn:naive_kfc} yields our KFC factorization of $\fisherMatApproxL{\layerIdx}$ into a Kronecker product of smaller factors:
\begin{equation}
\fisherMatApproxL{\layerIdx} = \convKronActL{\layerIdx-1} \otimes \convKronGradL{\layerIdx},
\end{equation}
where
\begin{align}
[\convKronActL{\layerIdx-1}]_{\inputMapIdx \numOffsets + \spatOffset, \, \inputMapIdx^\prime \numOffsets + \spatOffset^\prime} &\triangleq \boundaryFunction(\spatOffset, \spatOffset^\prime) \, \autoActivations(\inputMapIdx, \inputMapIdx^\prime, \spatOffset^\prime - \spatOffset) \nonumber \\
[\convKronActL{\layerIdx-1}]_{\inputMapIdx \numOffsets + \spatOffset, \, 0} = [\convKronActL{\layerIdx-1}]_{0, \, \inputMapIdx \numOffsets + \spatOffset} &\triangleq \boundaryFunctionUni(\spatOffset) \, \meanActivations(\inputMapIdx) \nonumber \\
[\convKronActL{\layerIdx-1}]_{0, \, 0} &\triangleq \numLocs \nonumber \\
[\convKronGradL{\layerIdx}]_{\outputMapIdx, \outputMapIdx^\prime} &\triangleq \autoGrad(\outputMapIdx, \outputMapIdx^\prime, 0). \label{eqn:conv_kron_formulas}
\end{align}
(We will derive much simpler formulas for $\convKronActL{\layerIdx-1}$ and $\convKronGradL{\layerIdx}$ in the next section.) Using this factorization, the rest of the K-FAC algorithm can be carried out without modification. For instance, we can compute the approximate natural gradient using a damped version of $\fisherMatApprox$ analogously to Eqns.~\ref{eqn:factored_damping} and \ref{eqn:approx_ng_damped} of Section~\ref{sec:kfac}:
\begin{align}
\fisherMatApproxGammaL{\gammaParam}{\layerIdx} &= \convKronActGammaL{\gammaParam}{\layerIdx-1} \otimes \convKronGradGammaL{\gammaParam}{\layerIdx} \\
&\triangleq \left( \convKronActL{\layerIdx-1} + \piParamL{\layerIdx} \sqrt{\lambdaParam + \gammaParam}\, \ident \right) \otimes \left( \convKronGradL{\layerIdx} + \frac{1}{\piParamL{\layerIdx}} \sqrt{\lambdaParam + \gammaParam}\, \ident \right). \label{eqn:factored_damping_conv} \\
\approxNatGrad = [\fisherMatApproxGamma{\gammaParam}]^{-1} \nabla \objective &= \begin{pmatrix} \kvec \left( \convKronGradGammaLInv{\gammaParam}{1} (\nabla_{\weightsBiasesL{1}} \objective) \convKronActGammaLInv{\gammaParam}{0} \right) \\ \vdots \\ \kvec \left( \convKronGradGammaLInv{\gammaParam}{\numLayers} (\nabla_{\weightsBiasesL{\numLayers}} \objective) \convKronActGammaLInv{\gammaParam}{\numLayers-1} \right) \end{pmatrix} \label{eqn:approx_ng_damped_conv}
\end{align}

Returning to our running example of AlexNet, $\weightsBiasesL{\layerIdx}$ is a $\numOutputMaps \times (\numInputMaps \numOffsets + 1) = 128 \times 1201$ matrix. Therefore the factors $\convKronActL{\layerIdx-1}$ and $\convKronGradL{\layerIdx}$ are $1201 \times 1201$ and $128 \times 128$, respectively. These matrices are small enough that they can be represented exactly and inverted in a reasonable amount of time, allowing us to efficiently compute the approximate natural gradient direction using Eqn.~\ref{eqn:approx_ng_damped_conv}.

\subsection{Estimating the factors}
\label{sec:estimating_factors}

Since the true covariance statistics are unknown, we estimate them empirically by sampling from the model's distribution, similarly to \citet{kfac}. To sample derivatives from the model's distribution, we select a mini-batch, sample the outputs from the model's predictive distribution, and backpropagate the derivatives. 

We need to estimate the Kronecker factors $\{\convKronActL{\layerIdx}\}_{\layerIdx=0}^{\numLayers-1}$ and $\{\convKronGradL{\layerIdx}\}_{\layerIdx=1}^{\numLayers}$. Since these matrices are defined in terms of the autocovariance functions $\autoActivations$ and $\autoGrad$, it would appear natural to estimate these functions empirically. Unfortunately, if the empirical autocovariances are plugged into Eqn.~\ref{eqn:conv_kron_formulas}, the resulting $\convKronActL{\layerIdx}$ may not be positive semidefinite. This is a problem, since negative eigenvalues in the approximate Fisher could cause the optimization to diverge (a phenomenon we have observed in practice). An alternative which at least guarantees PSD matrices is to simply ignore the boundary effects, taking $\boundaryFunction(\spatOffset, \spatOffset^\prime) = \boundaryFunction(\spatOffset) = \numLocs$ in Eqn.~\ref{eqn:conv_kron_formulas}. Sadly, we found this to give very inaccurate covariances, especially for higher layers, where the filters are of comparable size to the activation maps.

Instead, we estimate each $\convKronActL{\layerIdx}$ directly using the following fact:

\begin{theorem}
Under assumption {\bf SH},
\begin{align}
\convKronActL{\layerIdx} &= \expect \left[ \expansion{\activationsMatL{\layerIdx}}_\homog^\transpose \expansion{\activationsMatL{\layerIdx}}_\homog \right] \\
\convKronGradL{\layerIdx} &= \frac{1}{\numLocs} \expect \left[ \grad \preActivationsMatL{\layerIdx}^\transpose \grad \preActivationsMatL{\layerIdx} \right].
\end{align}
(The $\expansion{\cdot}$ notation is defined in Section~\ref{sec:conv_nets_efficient}.)
\end{theorem}

\begin{proof}
See Appendix \ref{app:proofs}.
\end{proof}

Using this result, we define the empirical statistics for a given mini-batch:
\begin{align}
\convKronActEmpL{\layerIdx} &= \frac{1}{\mbsize} \expansion{\activationsMatBatchL{\layerIdx}}_\homog^\transpose \expansion{\activationsMatBatchL{\layerIdx}}_\homog \nonumber \\
\convKronGradEmpL{\layerIdx} &= \frac{1}{\mbsize \numLocs} \grad \preActivationsMatBatchL{\layerIdx}^\transpose \grad \preActivationsMatBatchL{\layerIdx} \label{eqn:conv_kron_emp}
\end{align}
Since the estimates $\convKronActEmpL{\layerIdx} $ and $\convKronGradEmpL{\layerIdx}$ are computed in terms of matrix inner products, they are always PSD matrices.  Importantly, because $\expansion{\activationsMatBatchL{\layerIdx}}$ and $\grad \preActivationsMatBatchL{\layerIdx}$ are the same matrices used to implement the convolution operations (Section~\ref{sec:conv_nets_efficient}), the computation of covariance statistics enjoys the same memory locality properties as the convolution operations. 

At the beginning of training, we estimate $\{\convKronActL{\layerIdx}\}_{\layerIdx=0}^{\numLayers-1}$ and $\{\convKronGradL{\layerIdx}\}_{\layerIdx=1}^{\numLayers}$ from the full dataset (or a large subset) using Eqn.~\ref{eqn:conv_kron_emp}. Subsequently, we maintain exponential moving averages of these matrices, where these equations are applied to each mini-batch, i.e.
\begin{align}
\convKronActL{\layerIdx} &\gets \avgWeight \convKronActL{\layerIdx} + (1 - \avgWeight) \convKronActEmpL{\layerIdx} \nonumber \\
\convKronGradL{\layerIdx} &\gets \avgWeight \convKronGradL{\layerIdx} + (1 - \avgWeight) \convKronGradEmpL{\layerIdx}, \label{eqn:conv_kron_avg}
\end{align}
where $\avgWeight$ is a parameter which determines the timescale for the moving average.

\subsection{Using KFC in optimization}
\label{sec:optimization_method}

So far, we have defined an approximation $\fisherMatApproxGamma{\gammaParam}$ to the Fisher matrix $\fisherMat$ which can be tractably inverted. This can be used in any number of ways in the context of optimization, most simply by using $\approxNatGrad = [\fisherMatApproxGamma{\gammaParam}]^{-1} \nabla \objective$ as an approximation to the natural gradient $\fisherMat^{-1} \nabla \objective$. Alternatively, we could use it in the context of the full K-FAC algorithm, or as a preconditioner for iterative second-order methods \citep{HF,KSD,sfo}.

In our experiments, we explored two particular instantiations of KFC in optimization algorithms. First, in order to provide as direct a comparison as possible to standard SGD-based optimization, we used $\approxNatGrad$ in the context of a generic approximate natural gradient descent procedure; this procedure is like SGD, except that $\approxNatGrad$ is substituted for the Euclidean gradient. Additionally, we used momentum, update clipping, and parameter averaging --- all standard techniques in the context of stochastic optimization.\footnote{Our SGD baseline used momentum and parameter averaging as well. Clipping was not needed for SGD, for reasons explained in Appendix \ref{app:approx_ng}.} One can also view this as a preconditioned SGD method, where $\fisherMatApproxGamma{\gammaParam}$ is used as the preconditioner. Therefore, we refer to this method in our experiments as KFC-pre (to distinguish it from the KFC approximation itself). This method is spelled out in detail in Appendix \ref{app:approx_ng}.

We also explored the use of $\fisherMatApproxGamma{\gammaParam}$ in the context of K-FAC, which (in addition to the techniques of Section \ref{sec:kfac}), includes methods for adaptively changing the learning rate, momentum, and damping parameters over the course of optimization. The full algorithm is given in Appendix \ref{app:kfac}. Our aim was to measure how KFC can perform in the context of a sophisticated and well-tuned second-order optimization procedure. We found that the adaptation methods tended to choose stable values for the learning rate, momentum, and damping parameters, suggesting that these could be replaced with fixed values (as in KFC-pre). Since both methods performed similarly, we report results only for KFC-pre. We note that this finding stands in contrast with the autoencoder experiments of \citet{kfac}, where the adapted parameters varied considerably over the course of optimization.

With the exception of inverting the Kronecker factors, all of the heavy computation for our methods was performed on the GPU. We based our implementation on CUDAMat \citep{cudamat} and the convolution kernels provided by the Toronto Deep Learning ConvNet (TDLCN) package \citep{tdlcn}. Full details on our GPU implementation and other techniques for minimizing computational overhead are given in Appendix \ref{app:implementation}.

\section{Theoretical analysis}
\label{sec:theory}

\subsection{Invariance}
\label{sec:invariance}

Natural gradient descent is motivated partly by way of its invariance to reparameterization: regardless of how the model is parameterized, the updates are equivalent up to the first order. Approximations to natural gradient don't satisfy full invariance to parameterization, but certain approximations have been shown to be invariant to more limited, but still fairly broad, classes of transformations. \citet{ollivier_invariance} showed that one such approximation was invariant to (invertible) affine transformations of individual activations. This class of transformations includes replacing sigmoidal with tanh activation functions, as well as the centering transformations discussed in the next section. \citet{kfac} showed that K-FAC is invariant to a broader class of reparameterizations: affine transformations of the activations (considered as a group), both before and after the nonlinearity. In addition to affine transformations of individual activations, this class includes transformations which whiten the activations to have zero mean and unit covariance. The transformations listed here have all been used to improve optimization performance (see next section), so these invariance properties provide an interesting justification of approximations to natural gradient methods. I.e., to the extent that these transformations help optimization, approximate natural gradient descent methods can be expected to achieve such benefits automatically.

For convolutional layers, we cannot expect an algorithm to be invariant to arbitrary affine transformations of a given layer's activations, as such transformations can change the set of functions which are representable. (Consider for instance, a transformation which permutes the spatial locations.) However, we show that the KFC updates are invariant to homogeneous, \emph{pointwise} affine transformations of the activations, both before and after the nonlinearity. This is perhaps an overly limited statement, as it doesn't use the fact that the algorithm accounts for spatial correlations. However, it still accounts for a broad set of transformations, such as normalizing activations to be zero mean and unit variance either before or after the nonlinearity.

To formalize this, recall that a layer's activations are represented as a $\numLocs \times \numInputMaps$ matrix and are computed from that layer's pre-activations by way of an elementwise nonlinearity, i.e.~$\activationsMatL{\layerIdx} = \nonlinearityL{\layerIdx}(\preActivationsMatL{\layerIdx})$. We replace this with an activation function $\nonlinearityTransL{\layerIdx}$ which additionally computes affine transformations before and after the nonlinearity. Such transformations can be represented in matrix form:
\begin{equation}
\activationsMatTransL{\layerIdx} = \nonlinearityTransL{\layerIdx}(\preActivationsMatTransL{\layerIdx}) = \nonlinearityL{\layerIdx}(\preActivationsMatTransL{\layerIdx} \transMatPreL{\layerIdx} + \onesVec \transOffsetPreL{\layerIdx}^\transpose) \transMatPostL{\layerIdx} + \onesVec \transOffsetPostL{\layerIdx}^\transpose, \label{eqn:affine_transformation}
\end{equation}
where $\transMatPreL{\layerIdx}$ and $\transMatPostL{\layerIdx}$ are invertible matrices, and $\transOffsetPreL{\layerIdx}$ and $\transOffsetPostL{\layerIdx}$ are vectors. For convenience, the inputs to the network can be treated as an activation function $\nonlinearityL{0}$ which takes no arguments. We also assume the final layer outputs are not transformed, i.e.~$\transMatPostL{\numLayers} = \ident$ and $\transOffsetPostL{\numLayers} = \zeroMat$. KFC is invariant to this class of transformations:

\begin{theorem}
Let $\network$ be a network with parameter vector $\paramVec$ and activation functions $\{\nonlinearityL{\layerIdx}\}_{\layerIdx=0}^\numLayers$. Given activation functions $\{\nonlinearityTransL{\layerIdx}\}_{\layerIdx=0}^\numLayers$ defined as in Eqn.~\ref{eqn:affine_transformation}, there exists a parameter vector $\paramVecTrans$ such that a network $\networkTrans$ with parameters $\paramVecTrans$ and activation functions $\{\nonlinearityTransL{\layerIdx}\}_{\layerIdx=0}^\numLayers$ computes the same function as $\network$. The KFC updates on $\network$ and $\networkTrans$ are equivalent, in that the resulting networks compute the same function.
\label{thm:invariance}
\end{theorem}

\begin{proof}
See Appendix \ref{app:proofs}.
\end{proof}

Invariance to affine transformations also implies approximate invariance to smooth nonlinear transformations; see \citet{ng_martens} for further discussion.

\subsection{Relationship with other algorithms}
\label{sec:prong}

Other neural net optimization methods have been proposed which attempt to correct for various statistics of the activations or gradients. Perhaps the most commonly used are algorithms which attempt to adapt learning rates for individual parameters based on the variance of the gradients \citep{lecun_tricks,adagrad,rmsprop,ADADELTA,adam}. These can be thought of as diagonal approximations to the Hessian or the Fisher matrix.\footnote{Some of these methods use the \emph{empirical} Fisher matrix, which differs from the proper Fisher matrix in that the targets are taken from the training data rather than sampled from the model's predictive distribution. The empirical Fisher matrix is less closely related to the curvature than is the proper one \citep{ng_martens}.}

Another class of approaches attempts to reparameterize a network such that its activations have zero mean and unit variance, with the goals of preventing covariate shift and improving the conditioning of the curvature \citep{enhanced_gradient,Vatanen_centering,batch_normalization}. Centering can be viewed as an approximation to natural gradient where the Fisher matrix is approximated with a directed Gaussian graphical model \citep{FANG}. 
As discussed in Section \ref{sec:invariance}, KFC is invariant to re-centering of activations, so it ought to automatically enjoy the optimization benefits of centering. However, batch normalization \citep{batch_normalization} includes some effects not automatically captured by KFC. First, the normalization is done separately for each mini-batch rather than averaged across mini-batches; this introduces stochasticity into the computations which may serve as a regularizer. Second, it discourages large covariate shifts in the pre-activations, which may help to avoid dead units. Since batch normalization is better regarded as a modification to the architecture than an optimization algorithm, it can be combined with KFC; we investigated this in our experiments.

Projected Natural Gradient \citep[PRONG;][]{natural_neural_networks} goes a step further than centering methods by fully whitening the activations in each layer. In the case of fully connected layers, the activations are transformed to have zero mean and unit covariance. For convolutional layers, they apply a linear transformation that whitens the activations \emph{across feature maps}. While PRONG includes clever heuristics for updating the statistics, it's instructive to consider an idealized version of the method which has access to the exact statistics. We can interpret this idealized PRONG in our own framework as arising from following two additional approximations:
\begin{itemize}
\item {\bf Spatially uncorrelated activations (SUA).} The activations at any two distinct spatial locations are uncorrelated, \emph{i.e.}~$\Cov(\activationsJS{\inputMapIdx}{\spatIdx}, \activationsJS{\inputMapIdx^\prime}{\spatIdx^\prime}) = 0$ for $\spatIdx \neq \spatIdx^\prime$. Also assuming {\bf SH}, the correlations can then be written as $\Cov(\activationsJS{\inputMapIdx}{\spatIdx}, \activationsJS{\inputMapIdx^\prime}{\spatIdx}) = \autoCovActivations(\inputMapIdx, \inputMapIdx^\prime)$. 
\item {\bf White derivatives (WD).} Pre-activation derivatives are uncorrelated and have spherical covariance, i.e.~$\autoGrad(\outputMapIdx, \outputMapIdx^\prime, \spatOffset) \propto \kronDelta{\outputMapIdx}{\outputMapIdx^\prime} \kronDelta{\spatOffset}{0}$. We can assume WLOG that the proportionality constant is 1, since any scalar factor can be absorbed into the learning rate.
\end{itemize}

\begin{theorem}
Combining approximations {\bf IAD}, {\bf SH}, {\bf SUA}, and {\bf WD} results in the following approximation to the entries of the Fisher matrix:
\begin{equation}
\expect \left[ \gradWeightsIJS{\outputMapIdx}{\inputMapIdx}{\spatOffset} \gradWeightsIJS{\outputMapIdx^\prime}{\inputMapIdx^\prime}{\spatOffset^\prime} \right] = \boundaryFunction(\spatOffset, \spatOffset^\prime) \, \autoActivationsNNN(\inputMapIdx, \inputMapIdx^\prime, \spatOffset^\prime - \spatOffset) \, \kronDelta{\outputMapIdx}{\outputMapIdx^\prime},
\end{equation}
where $\indicator$ is the indicator function and $\autoActivationsNNN(\inputMapIdx, \inputMapIdx^\prime, \spatOffset) \triangleq \autoCovActivations(\inputMapIdx, \inputMapIdx^\prime) \kronDelta{\spatOffset}{0} + \meanActivations(\inputMapIdx) \meanActivations(\inputMapIdx^\prime)$ is the uncentered autocovariance function. ($\boundaryFunction$ is defined in Theorem \ref{thm:kfc_sud}. Formulas for the remaining entries are given in Appendix~\ref{app:proofs}.) If the $\boundaryFunction(\spatOffset, \spatOffset^\prime)$ term is dropped, the resulting approximate natural gradient descent update rule is equivalent to idealized PRONG, up to rescaling.
\end{theorem}

As we later discuss in Section~\ref{sec:eval_assumptions}, assumption {\bf WD} appears to hold up well empirically, while {\bf SUA} appears to lose a lot of information. Observe, for instance, that the input images are themselves treated as a layer of activations. Assumption {\bf SUA} amounts to modeling each channel of an image as white noise, corresponding to a flat power spectrum. Images have a well-characterized $1/f^p$ power spectrum with $p \approx 2$ \citep{image_statistics_review}, which implies that the curvature may be much larger in directions corresponding to low-frequency Fourier components than in directions corresponding to high-frequency components.

\section{Experiments}

We have evaluated our method on two standard image recognition benchmark datasets: CIFAR-10 \citep{cifar}, and Street View Housing Numbers \citep[SVHN;][]{svhn}. Our aim is not to achieve state-of-the-art performance, but to evaluate KFC's ability to optimize previously published architectures. We first examine the probabilistic assumptions, and then present optimization results.

For CIFAR-10, we used the architecture from \verb+cuda-convnet+\footnote{\url{https://code.google.com/p/cuda-convnet/}} which achieved 18\% error in 20 minutes. This network consists of three convolution layers and a fully connected layer. (While \verb+cuda-convnet+ provides some better-performing architectures, we could not use these, since these included locally connected layers, which KFC can't handle.) For SVHN, we used the architecture of \citet{nitish_thesis}. This architecture consists of three convolutional layers followed by three fully connected layers, and uses dropout for regularization. Both of these architectures were carefully tuned for their respective tasks. Furthermore, the TDLCN CUDA kernels we used were carefully tuned at a low level to implement SGD updates efficiently for both of these architectures. Therefore, we believe our SGD baseline is quite strong.

\subsection{Evaluating the probabilistic modeling assumptions}
\label{sec:eval_assumptions}

\begin{figure}
\begin{center}
(a)\includegraphics[width=0.20 \textwidth]{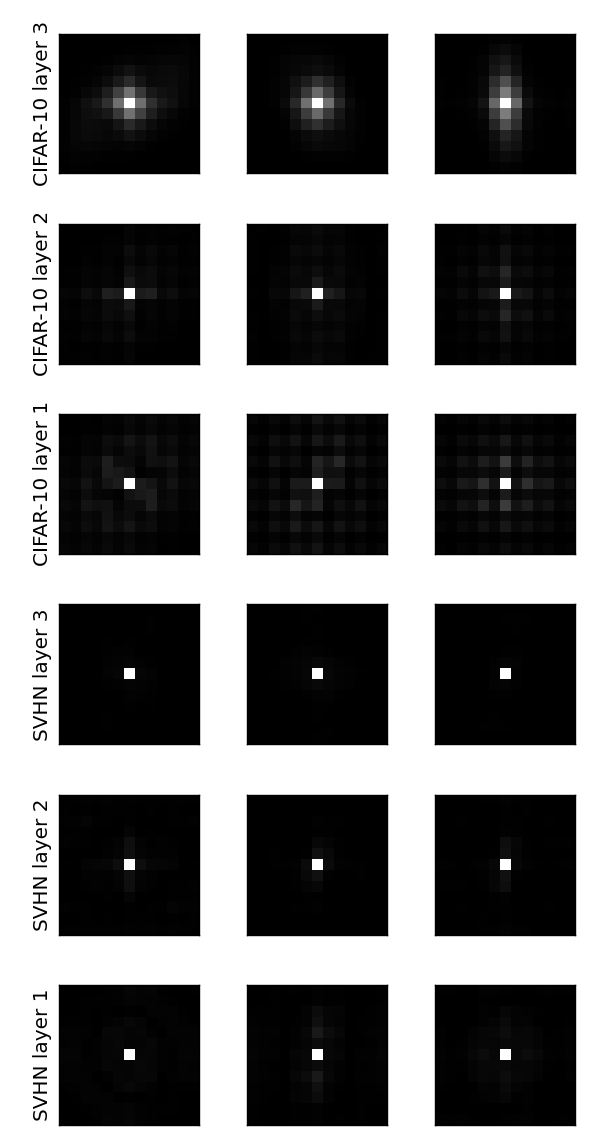}
(b)\includegraphics[width=0.22 \textwidth]{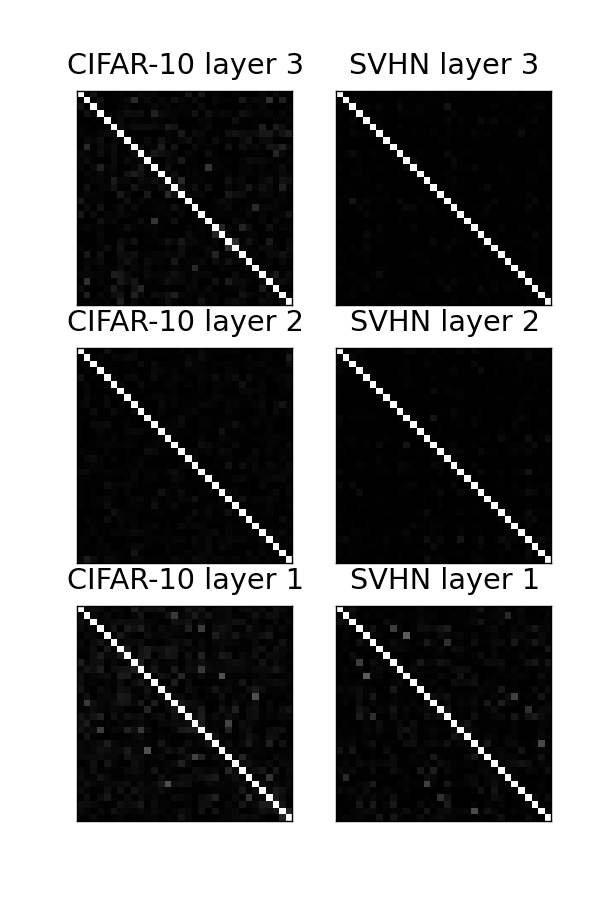}
(c)\includegraphics[width=0.20 \textwidth]{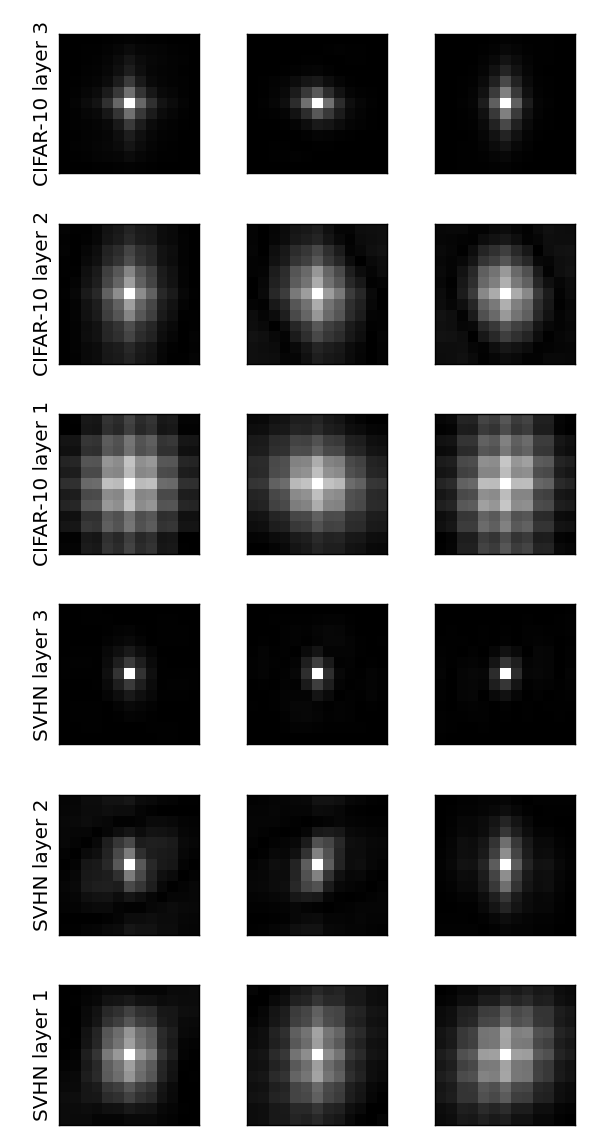}
(d)\includegraphics[width=0.22 \textwidth]{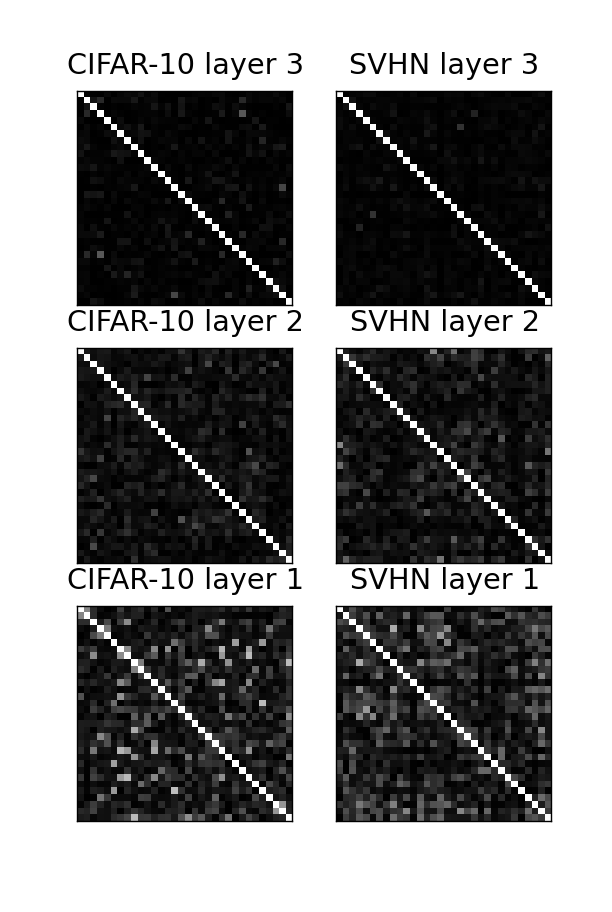}
\end{center}
\caption{Visualization of the absolute values of the correlations between the pre-activation derivatives for all of the convolution layers of CIFAR-10 and SVHN networks trained with SGD. {\bf (a)} Autocorrelation functions of the derivatives of three feature maps from each layer. {\bf (b)} Cross-map correlations for a single spatial position. {\bf (c, d)} Same as (a) and (b), except that the networks use average pooling rather than max-pooling.}
\label{fig:gradient_autocorrelations}
\end{figure}

In defining KFC, we combined three probabilistic modeling assumptions: independent activations and derivatives {\bf (IAD)}, spatial homogeneity {\bf (SH)}, and spatially uncorrelated derivatives {\bf (SUD)}. As discussed above, {\bf IAD} is the same approximation made by standard K-FAC, and it was investigated in detail both theoretically and empirically by \cite{kfac}. One implicitly assumes {\bf SH} when choosing to use a convolutional architecture. However, {\bf SUD} is perhaps less intuitive. Why should we suppose the derivatives are spatially uncorrelated?  Conversely, why not go a step further and assume the \emph{activations} are spatially uncorrelated (as does PRONG; see Section~\ref{sec:prong}) or even drop all of the correlations (thereby obtaining a much simpler diagonal approximation to the Fisher matrix)?

We investigated the autocorrelation functions for networks trained on CIFAR-10 and SVHN, each with 50 epochs of SGD. (These models were trained long enough to achieve good test error, but not long enough to overfit.) Derivatives were sampled from the model's distribution as described in Section \ref{sec:kfac}. Figure~\ref{fig:gradient_autocorrelations}(a) shows the autocorrelation functions of the pre-activation gradients for three (arbitrary) feature maps in all of the convolution layers of both networks. Figure~\ref{fig:gradient_autocorrelations}(b) shows the correlations between derivatives for different feature maps in the same spatial position. Evidently, the derivatives are very weakly correlated, both spatially and cross-map, although there are some modest cross-map correlations in the first layers of both models, as well as modest spatial correlations in the top convolution layer of the CIFAR-10 network. This suggests that {\bf SUD} is a good approximation for these networks.

Interestingly, the lack of correlations between derivatives appears to be a result of max-pooling. Max-pooling has a well-known sparsifying effect on the derivatives, as any derivative is zero unless the corresponding activation achieves the maximum within its pooling group. Since neighboring locations are unlikely to simultaneously achieve the maximum, max-pooling weakens the spatial correlations. To test this hypothesis, we trained networks equivalent to those described above, except that the max-pooling layers were replaced with average pooling. The spatial autocorrelations and cross-map correlations are shown in Figure~\ref{fig:gradient_autocorrelations}(c, d). Replacing max-pooling with average pooling dramatically strengthens both sets of correlations.

\begin{figure}
\begin{center}
(a)\includegraphics[width=0.20 \textwidth]{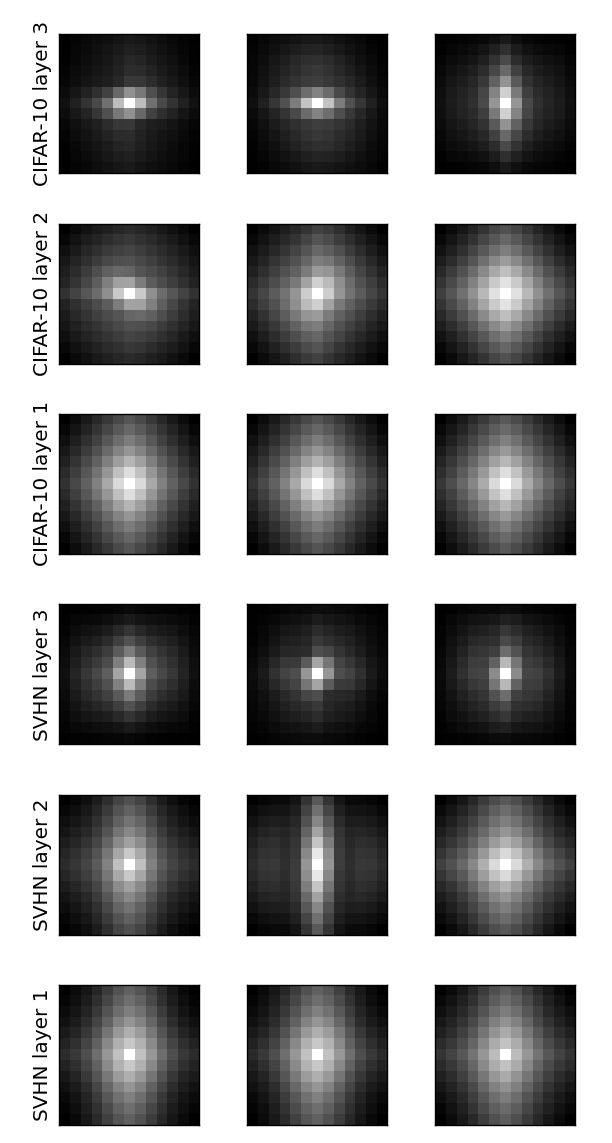}
(b)\includegraphics[width=0.30 \textwidth]{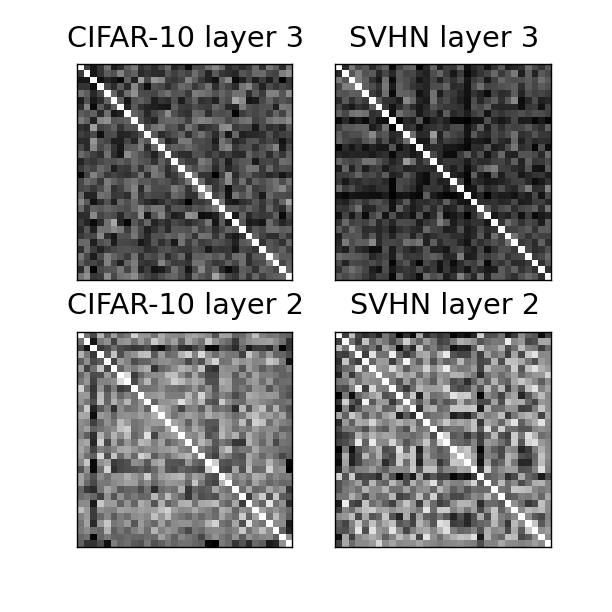}
\end{center}
\caption{Visualization of the uncentered correlations $\autoActivations$ between activations in all of the convolution layers of the CIFAR-10 and SVHN networks. {\bf (a)} Spatial autocorrelation functions of three feature maps in each layer. {\bf (b)} Correlations of the activations at a given spatial location. The activations have much stronger correlations than the backpropagated derivatives.}
\label{fig:activation_autocorrelations}
\end{figure}

In contrast with the derivatives, the activations have very strong correlations, both spatially and cross-map, as shown in Figure~\ref{fig:activation_autocorrelations}. This suggests the spatially uncorrelated activations {\bf (SUA)} assumption implicitly made by some algorithms could be problematic, despite appearing superficially analogous to {\bf SUD}.

\subsection{Optimization performance}
\label{sec:optimization_performance}

\begin{figure}
\begin{center}
(a)\includegraphics[width=0.45 \textwidth]{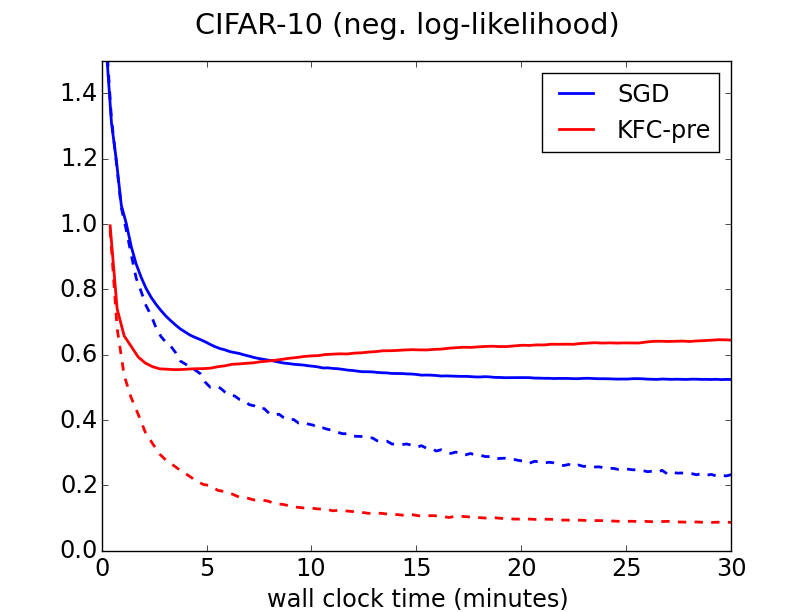}
(b)\includegraphics[width=0.45 \textwidth]{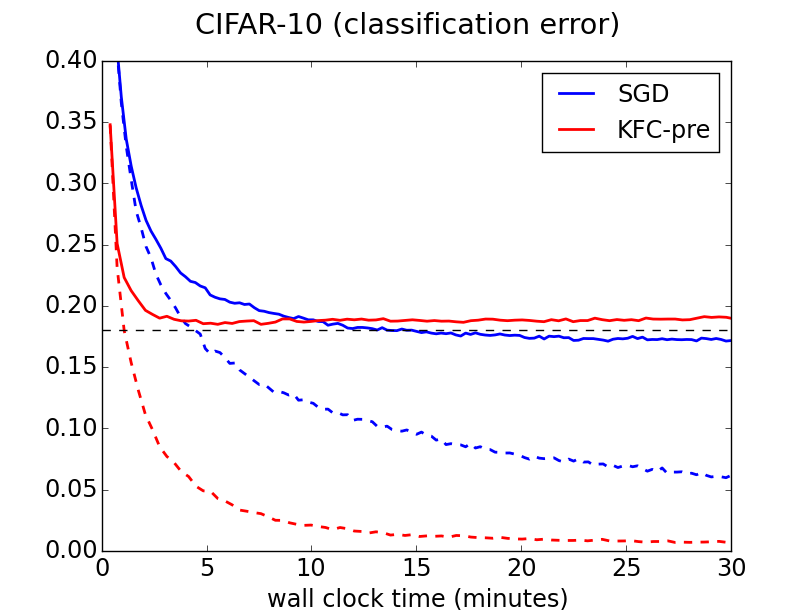} \\
(c)\includegraphics[width=0.45 \textwidth]{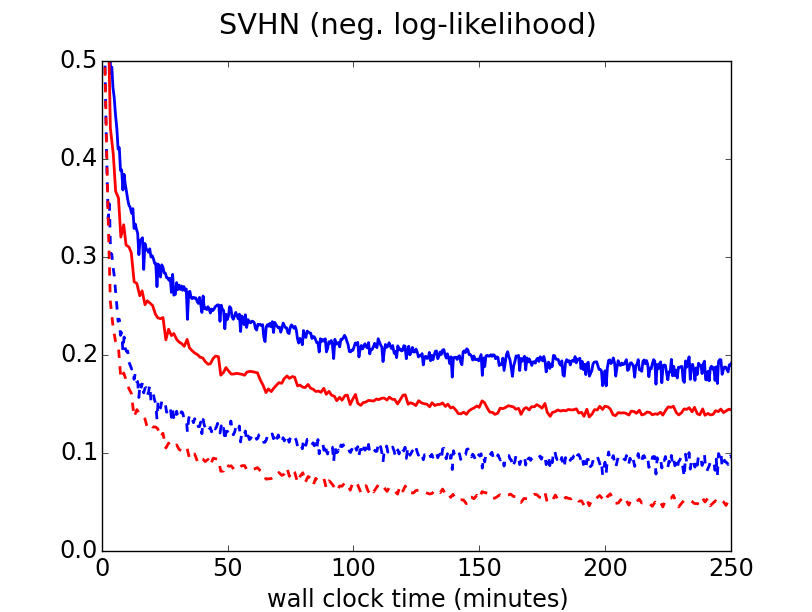}
(d)\includegraphics[width=0.45 \textwidth]{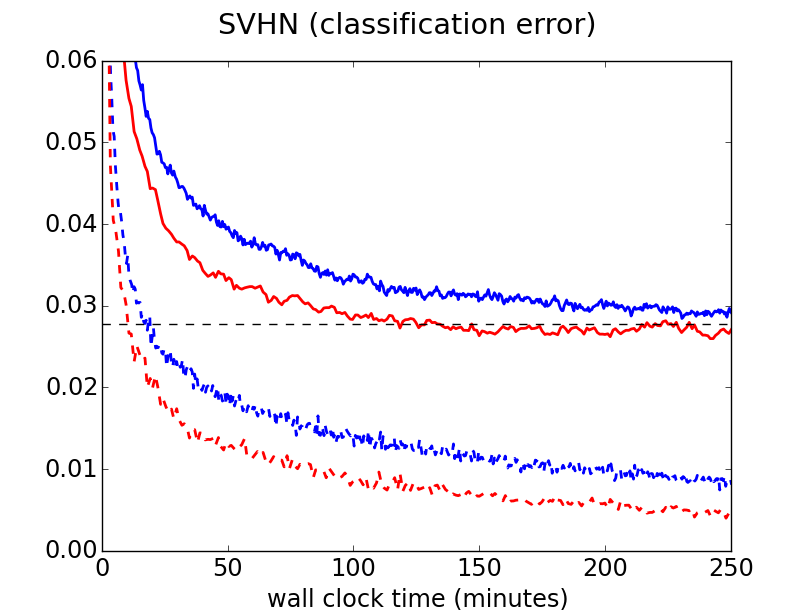}
\end{center}
\caption{Optimization performance of KFC-pre and SGD. {\bf (a)} CIFAR-10, negative log-likelihood. {\bf (b)} CIFAR-10, classification error. {\bf (c)} SVHN, negative log-likelihood. {\bf (d)} SVHN, classification error. {\bf Solid lines} represent test error and {\bf dashed lines} represent training error. The {\bf horizontal dashed line} represents the previously reported test error for the same architecture.}
\label{fig:optimization}
\end{figure}

We evaluated KFC-pre in the context of optimizing deep convolutional networks. We compared against stochastic gradient descent (SGD) with momentum, which is widely considered a strong baseline for training conv nets. All architectural choices (e.g.~sizes of layers) were kept consistent with the previously published configurations. Since the focus of this work is optimization rather than generalization, metaparameters were tuned with respect to \emph{training} error. This protocol was favorable to the SGD baseline, as the learning rates which performed the best on training error also performed the best on test error.\footnote{For KFC-pre, we encountered a more significant tradeoff between training and test error, most notably in the choice of mini-batch size, so the presented results do not reflect our best runs on the test set. For instance, as reported in Figure \ref{fig:optimization}, the test error on CIFAR-10 leveled off at 18.5\% after 5 minutes, after which the network started overfitting. When we reduced the mini-batch size from 512 to 128, the test error reached 17.5\% after 5 minutes and 16\% after 35 minutes. However, this run performed far worse on the training set. On the flip side, very large mini-batch sizes hurt generalization for both methods, as discussed in Section~\ref{sec:distributed}.} For both SGD and KFC-pre, we tuned the learning rates from the set $\{0.3, 0.1, 0.03, \ldots, 0.0003\}$ separately for each experiment.
For KFC-pre, we also chose several algorithmic parameters using the method of Appendix~\ref{app:implementation}, which considers only per-epoch running time and not final optimization performance.\footnote{For SGD, we used a momentum parameter of 0.9 and mini-batches of size 128, which match the previously published configurations. For KFC-pre, we used a momentum parameter of 0.9, mini-batches of size 512, and a damping parameter $\gammaParam = 10^{-3}$. In both cases, our informal explorations did not find other values which performed substantially better in terms of training error.}

For both SGD and KFC-pre, we used an exponential moving average of the iterates (see Appendix~\ref{app:approx_ng}) with a timescale of 50,000 training examples (which corresponds to one epoch on CIFAR-10). This helped both SGD and KFC-pre substantially. All experiments for which wall clock time is reported were run on a single Nvidia GeForce GTX Titan Z GPU board.

As baselines, we also tried Adagrad \citep{adagrad}, RMSProp \citep{rmsprop}, and Adam \citep{adam}, but none of these approaches outperformed carefully tuned SGD with momentum. This is consistent with the observations of \citet{adam}.

Figure~\ref{fig:optimization}(a,b) shows the optimization performance on the CIFAR-10 dataset, in terms of wall clock time. Both KFC-pre and SGD reached approximately the previously published test error of 18\% before they started overfitting. However, KFC-pre reached 19\% test error in 3 minutes, compared with 9 minutes for SGD. The difference in training error was more significant: KFC-pre reaches a training error of 6\% in 4 minutes, compared with 30 minutes for SGD. On SVHN, KFC-pre reached the previously published test error of 2.78\% in 120 minutes, while SGD did not reach it within 250 minutes. (As discussed above, test error comparisons should be taken with a grain of salt because algorithms were tuned based on training error; however, any biases introduced by our protocol would tend to favor the SGD baseline over KFC-pre.)

\begin{figure}
\begin{center}
(a)\includegraphics[width=0.45 \textwidth]{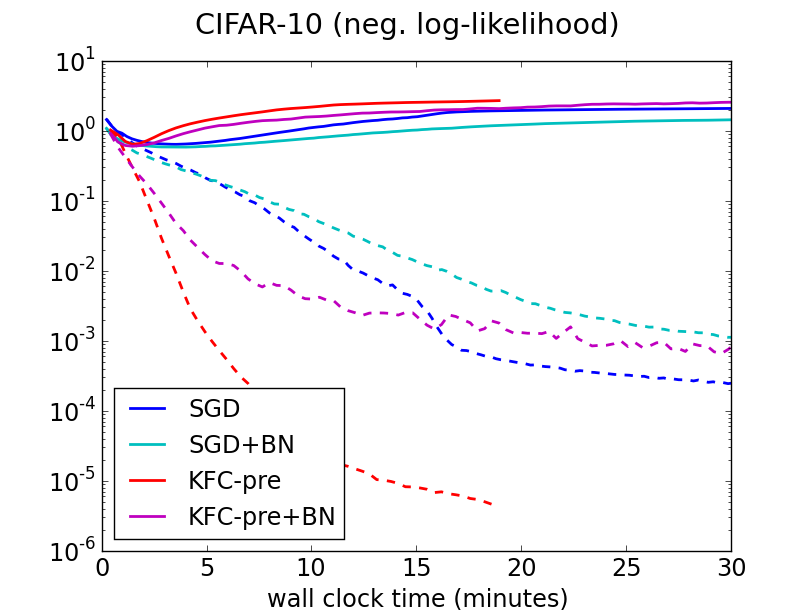}
(b)\includegraphics[width=0.45 \textwidth]{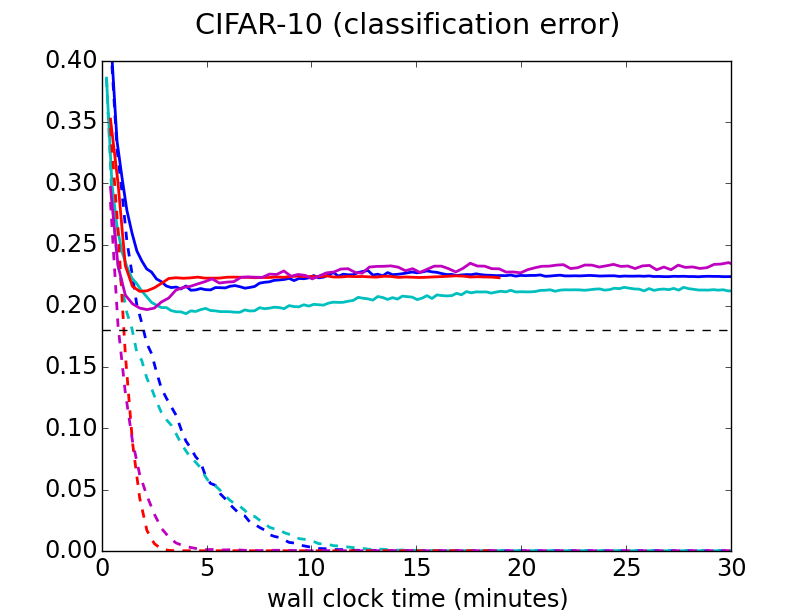} 
\end{center}
\caption{Optimization performance of KFC-pre and SGD on a CIFAR-10 network, with and without batch normalization (BN). {\bf (a)} Negative log-likelihood, on a log scale. {\bf (b)} Classification error. {\bf Solid lines} represent test error and {\bf dashed lines} represent training error. The {\bf horizontal dashed line} represents the previously reported test error for the same architecture. The KFC-pre training curve is cut off because the algorithm became unstable when the training NLL reached $4 \times 10^{-6}$.}
\label{fig:bnorm}
\end{figure}

Batch normalization \citep[BN][]{batch_normalization} has recently had much success at training a variety of neural network architectures. It has been motivated both in terms of optimization benefits (because it reduces covariate shift) and regularization benefits (because it adds stochasticity to the updates). However, BN is best regarded not as an optimization algorithm, but as a modification to the network architecture, and it can be used in conjunction with algorithms other than SGD. We modified the original CIFAR-10 architecture to use batch normalization in each layer. Since the parameters of a batch normalized network would have a different scale from those of an ordinary network, we disabled the $\ell_2$ regularization term so that both networks would be optimized to the same objective function. While our own (inefficient) implementation of batch normalization incurred substantial computational overhead, we believe an efficient implementation ought to have very little overhead; therefore, we simulated an efficient implementation by reusing the timing data from the non-batch-normalized networks. Learning rates were tuned separately for all four conditions (similarly to the rest of our experiments).

Training curves are shown in Figure \ref{fig:bnorm}. All of the methods achieved worse test error than the original network as a result of $\ell_2$ regularization being eliminated. However, the BN networks reached a lower test error than the non-BN networks before they started overfitting, consistent with the stochastic regularization interpretation of BN.\footnote{Interestingly, the BN networks were slower to optimize the training error than their non-BN counterparts. We speculate that this is because (1) the SGD baseline, being carefully tuned, didn't exhibit the pathologies that BN is meant to correct for (i.e.~dead units and extreme covariate shift), and (2) the regularization effects of BN made it harder to overfit.} For both the BN and non-BN architectures, KFC-pre optimized both the training and test error and NLL considerably faster than SGD. Furthermore, it appeared not to lose the regularization benefit of BN. This suggests that KFC-pre and BN can be combined synergistically.

\subsection{Potential for distributed implementation}
\label{sec:distributed}

\begin{figure}
\begin{center}
(a) \includegraphics[width=0.45 \textwidth]{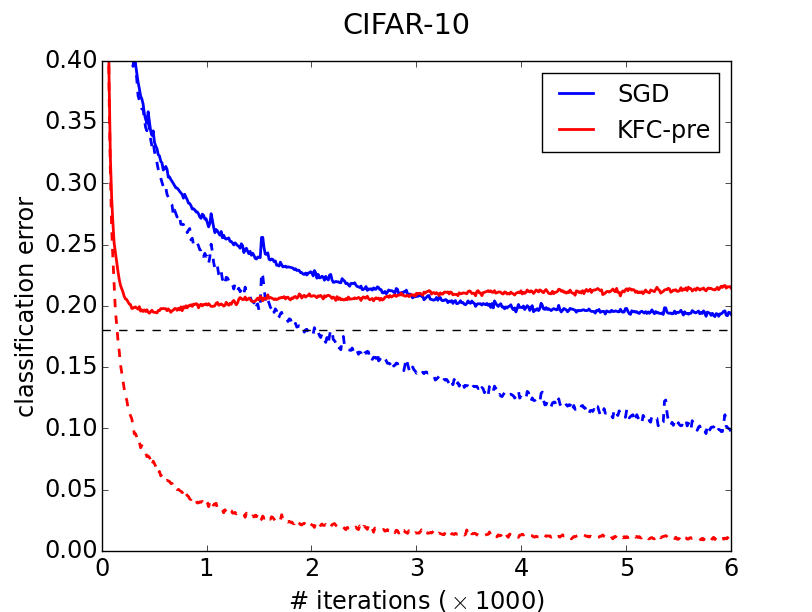}
(b) \includegraphics[width=0.45 \textwidth]{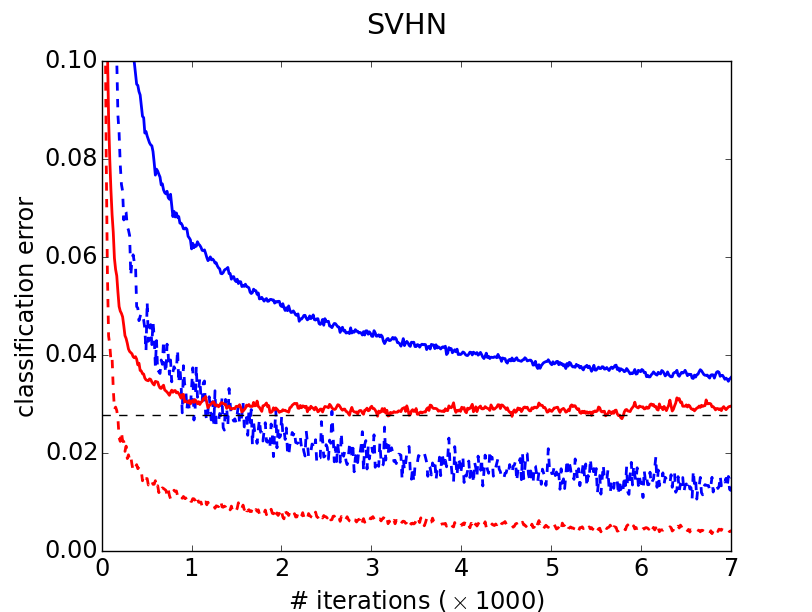}
\end{center}
\caption{Classification error as a function of the number of iterations (weight updates). Heuristically, this is a rough measure of how the algorithms might perform in a highly distributed setting. The $y$-axes represent classification error. {\bf (a)} CIFAR-10. {\bf (b)} SVHN. {\bf Solid lines} represent test error and {\bf dashed lines} represent training error. The {\bf horizontal dashed line} represents the previously reported test error for the same architecture.}
\label{fig:iterations}
\end{figure}

Much work has been devoted recently to highly parallel or distributed implementations of neural network optimization (e.g.~\citet{distributed_deep_nets}). Synchronous SGD effectively allows one to use very large mini-batches efficiently, which helps optimization by reducing the variance in the stochastic gradient estimates. However, the per-update performace levels off to that of batch SGD once the variance is no longer significant and curvature effects come to dominate. Asynchronous SGD partially alleviates this issue by using new network parameters as soon as they become available, but needing to compute gradients with stale parameters limits the benefits of this approach.

As a proxy for how the algorithms are likely to perform in a highly distributed setting\footnote{Each iteration of KFC-pre requires many of the same computations as SGD, most notably computing activations and gradients. There were two major sources of additional overhead: maintaining empirical averages of covariance statistics, and computing inverses or eigendecompositions of the Kronecker factors. These additional operations can almost certainly be performed asynchronously; in our own experiments, we only periodically performed these operations, and this did not cause a significant drop in performance. Therefore, we posit that each iteration of KFC-pre requires a comparable number of sequential operations to SGD for each weight update. This is in contrast to other methods which make good use of large mini-batches such as Hessian-free optimization \citep{HF}, which requires many sequential iterations for each weight update. KFC-pre also adds little communication overhead, as the Kronecker factors need not be sent to the worker nodes which compute the gradients.}, we measured the classification error as a function of the \emph{number of iterations} (weight updates) for each algorithm. Both algorithms were run with large mini-batches of size 4096 (in place of 128 for SGD and 512 for KFC-pre). Figure~\ref{fig:iterations} shows training curves for both algorithms on CIFAR-10 and SVHN, using the same architectures as above.\footnote{Both SGD and KFC-pre reached a slightly worse test error before they started overfitting, compared with the small-minibatch experiments of the previous section. This is because large mini-batches lose the regularization benefit of stochastic gradients. One would need to adjust the regularizer in order to get good generalization performance in this setting.} KFC-pre required far fewer weight updates to achieve good training and test error compared with SGD. For instance, on CIFAR-10, KFC-pre obtained a training error of 10\% after 300 updates, compared with 6000 updates for SGD, a 20-fold improvement. Similar speedups were obtained on test error and on the SVHN dataset. These results suggest that a distributed implementation of KFC-pre has the potential to obtain large speedups over distributed SGD-based algorithms.

\begin{small}
\bibliography{convnet}
\bibliographystyle{convnet}
\end{small}

\appendix

\section{Optimization methods}

\subsection{KFC as a preconditioner for SGD}
\label{app:approx_ng}

\begin{algorithm}
\begin{algorithmic}
  \REQUIRE \begin{varwidth}[t]{\linewidth}
    initial network parameters $\paramVecT{0}$ \par
    weight decay penalty $\lambdaParam$ \par
    learning rate $\learningRate$ \par
    momentum parameter $\momentumParam$ (suggested value: 0.9) \par
    parameter averaging timescale $\avgTimescale$ (suggested value: number of mini-batches in the dataset) \par
    damping parameter $\gammaParam$ (suggested value: $10^{-3}$, but this may require tuning) \par
    statistics update period $\updateStatsEvery$ (see Appendix~\ref{app:implementation}) \par
    inverse update period $\updateFacEvery$ (see Appendix~\ref{app:implementation}) \par
    clipping parameter $\clipParam$ (suggested value: 0.3) \par
    \end{varwidth}
  \STATE $\iterCount \gets 0$
  \STATE $\momentumVec \gets \zeroVec$
  \STATE $\avgWeight \gets e^{-1/\avgTimescale}$
  \STATE $\paramVecAvgT{0} \gets \paramVecT{0}$
  \STATE Estimate the factors $\{ \convKronActL{\layerIdx} \}_{\layerIdx=0}^{\numLayers-1}$ and $\{ \convKronGradL{\layerIdx} \}_{\layerIdx=1}^\numLayers$ on the full dataset using Eqn.~\ref{eqn:conv_kron_emp}
  \STATE Compute the inverses $\{ \convKronActGammaLInv{\gammaParam}{\layerIdx} \}_{\layerIdx=0}^{\numLayers-1}$ and $\{ \convKronGradGammaLInv{\gammaParam}{\layerIdx} \}_{\layerIdx=1}^\numLayers$ using Eqn.~\ref{eqn:factored_damping_conv}
  \WHILE{stopping criterion not met}
    \STATE $\iterCount \gets \iterCount + 1$
    \STATE Select a new mini-batch
    \STATE
    \IF{$\iterCount \equiv 0 \pmod{\updateStatsEvery}$}
      \STATE Update the factors $\{ \convKronActL{\layerIdx} \}_{\layerIdx=0}^{\numLayers-1}$ and $\{ \convKronGradL{\layerIdx} \}_{\layerIdx=1}^\numLayers$ using Eqn.~\ref{eqn:conv_kron_avg}
    \ENDIF
    \IF{$\iterCount \equiv 0 \pmod{\updateFacEvery}$}
      \STATE Compute the inverses $\{ \convKronActGammaLInv{\gammaParam}{\layerIdx} \}_{\layerIdx=0}^{\numLayers-1}$ and $\{ \convKronGradGammaLInv{\gammaParam}{\layerIdx} \}_{\layerIdx=1}^\numLayers$ using Eqn.~\ref{eqn:factored_damping_conv}
    \ENDIF
    \STATE
    \STATE Compute $\nabla \objective$ using backpropagation
    \STATE Compute $\approxNatGrad = [\fisherMatApproxGamma{\gammaParam}]^{-1} \nabla \objective$ using Eqn.~\ref{eqn:approx_ng_damped_conv}
    \STATE $\updateVec \gets -\learningRate \approxNatGrad$
    \STATE
    \STATE
    \COMMENT{Clip the update if necessary}
    \STATE Estimate $\updateNorm = \updateVec^\transpose \fisherMat \updateVec + \lambdaParam \updateVec^\transpose \updateVec$ using a subset of the current mini-batch
    \IF{$\updateNorm > \clipParam$}
      \STATE $\updateVec \gets \updateVec / \sqrt{\updateNorm / \clipParam}$
    \ENDIF
    \STATE
    \STATE $\momentumVecT{\iterCount} \gets \momentumParam \momentumVecT{\iterCount-1} + \updateVec$
    \COMMENT{Update momentum}
    \STATE $\paramVecT{\iterCount} \gets \paramVecT{\iterCount - 1} + \momentumVecT{\iterCount}$
    \COMMENT{Update parameters}
    \STATE $\paramVecAvgT{\iterCount} \gets \avgWeight \paramVecAvgT{\iterCount-1} + (1 - \avgWeight) \paramVecT{\iterCount}$
    \COMMENT{Parameter averaging}
  \ENDWHILE
  \RETURN Averaged parameter vector $\paramVecAvgT{\iterCount}$
  
\end{algorithmic}
\caption{Using KFC as a preconditioner for SGD}
\label{alg:kfc}
\end{algorithm}

The first optimization procedure we used in our experiments was a generic natural gradient descent approximation, where $\fisherMatApproxGamma{\gammaParam}$ was used to approximate $\fisherMat$. This procedure is like SGD with momentum, except that $\approxNatGrad$ is substituted for the Euclidean gradient. One can also view this as a preconditioned SGD method, where $\fisherMatApproxGamma{\gammaParam}$ is used as the preconditioner. To distinguish this optimization procedure from the KFC approximation itself, we refer to it as KFC-pre. Our procedure is perhaps more closely analogous to earlier Kronecker product-based natural gradient approximations \citep{heskes,povey_ng} than to K-FAC itself.

In addition, we used a variant of gradient clipping \citep{rnn_difficulty} to avoid instability. In particular, we clipped the approximate natural gradient update $\paramUpdate$ so that $\updateNorm \triangleq \paramUpdate^\transpose \fisherMat \paramUpdate < 0.3$, where $\fisherMat$ is estimated using 1/4 of the training examples from the current mini-batch. One motivation for this heuristic is that $\updateNorm$ approximates the KL divergence of the predictive distributions before and after the update, and one wouldn't like the predictive distributions to change too rapidly. The value $\updateNorm$ can be computed using curvature-vector products \citep{schraudolph}. The clipping was only triggered near the beginning of optimization, where the parameters (and hence also the curvature) tended to change rapidly.\footnote{This may be counterintuitive, since SGD applied to neural nets tends to take small steps early in training, at least for commonly used initializations. For SGD, this happens because the initial parameters, and hence also the initial curvature, are relatively small in magnitude. Our method, which corrects for the curvature, takes larger steps early in training, when the error signal is the largest.} Therefore, one can likely eliminate this step by initializing from a model partially trained using SGD.

Taking inspiration from Polyak averaging \citep{polyak_averaging,kevin}, we used an exponential moving average of the iterates. This helps to smooth out the variability caused by the mini-batch selection. The full optimization procedure is given in Algorithm~\ref{alg:kfc}.

\subsection{Kronecker-factored approximate curvature}
\label{app:kfac}

The central idea of K-FAC is the combination of approximations to the Fisher matrix described in Section \ref{sec:kfac}.  While one could potentially perform standard natural gradient descent using the approximate natural gradient $\approxNatGrad$, perhaps with a fixed learning rate and with fixed Tikhonov-style damping/reglarization, \citet{kfac} found that the most effective way to use $\approxNatGrad$ was within a robust 2nd-order optimization framework based on adaptively damped quadratic models, similar to the one employed in HF \citep{HF}.  In this section, we describe the K-FAC method in detail, while omitting certain aspects of the method which we do not use, such as the block tri-diagonal inverse approximation.

K-FAC uses a quadratic model of the objective to dynamically choose the step size $\LR$ and momentum decay parameter $\momDecayParam$ at each step.  This is done by taking $\paramUpdate = \LR \approxNatGrad + \momDecayParam \paramUpdate_{prev}$ where $\paramUpdate_{prev}$ is the update computed at the previous iteration, and minimizing the following quadratic model of the objective (over the current mini-batch):
\begin{equation}
\quadModel(\paramVec + \paramUpdate) = \objective(\paramVec) + \nabla \objective^\transpose \paramUpdate + \frac{1}{2} \paramUpdate^\transpose ( \fisherMat + \weightDecayParam \ident ) \paramUpdate. \label{eqn:loss_quad}
\end{equation}
where we assume the $\objective$ is the expected loss plus an $\ell_2$-regularization term of the form $\frac{\weightDecayParam}{2} \|\paramVec\|^2$.
Since $\fisherMat$ behaves like a curvature matrix, this quadratic function is similar to the second-order Taylor approximation to $\objective$. Note that here we use the \emph{exact} $\fisherMat$ for the mini-batch, rather than the approximation $\fisherMatApprox$.  Intuitively, one can think of $\paramUpdate$ as being itself iteratively optimized at each step in order to better minimize $\quadModel$, or in other words, to more closely match the true natural gradient (which is the exact minimum of $\quadModel$).  Interestingly, in full batch mode, this method is equivalent to performing preconditioned conjugate gradient in the vicinity of a local optimum (where $\fisherMat$ remains approximately constant).

To see how this minimization over $\LR$ and $\momDecayParam$ can be done efficiently, without computing the entire matrix $\fisherMat$, consider the general problem of minimizing $\quadModel$ on the subspace spanned by arbitrary vectors $\{\stepVecI{1}, \ldots, \stepVecI{\numStepVec}\}$. (In our case, $\numStepVec = 2$, $\stepVecI{1} = \approxNatGrad$ and $\stepVecI{2} = \paramUpdate_{prev}$.) The coefficients $\stepCoeffs$ can be found by solving the linear system $\stepA \stepCoeffs = -\stepB$, where $\stepA_{ij} = \stepVecI{i}^\transpose \fisherMat \stepVecI{j}$ and $\stepB_i = \nabla \objective^\transpose \stepVecI{i}$. To compute the matrix $\stepA$, we compute each of the matrix-vector products $\fisherMat \stepVecI{j}$ using automatic differentiation \citep{schraudolph}.

Both the approximate natural gradient $\approxNatGrad$ and the update $\paramUpdate$ (generated as described above) arise as the minimum, or approximate minimum, of a corresponding quadratic model.  In the case of $\paramUpdate$, this model is given by $\quadModel$ and is designed to be a good local approximation to the objective $\objective$.  Meanwhile, the quadratic model which is implicitly minimized when computing $\approxNatGrad$ is designed to approximate $\quadModel$ (by approximating $\fisherMat$ with $\fisherMatApprox$).

Because these quadratic models are approximations, naively minimizing them over $\Real^\numParams$ can lead to poor results in both theory and practice.  To help deal with this problem K-FAC employs an adaptive Tikhonov-style damping scheme applied to each of them (the details of which differ in either case).

To compensate for the inaccuracy of $\quadModel$ as a model of $\objective$, K-FAC adds a Tikhonov regularization term $\frac{\lambdaParam}{2} \|\paramUpdate\|^2$ to $\quadModel$ which encourages the update to remain small in magnitude, and thus more likely to remain in the region where $\quadModel$ is a reasonable approximation to $\objective$. This amounts to replacing $\weightDecayParam$ with $\weightDecayParam + \lambdaParam$ in Eqn.~\ref{eqn:loss_quad}. Note that this technique is formally equivalent to performing constrained minimization of $\quadModel$ within some spherical region around $\paramUpdate = 0$ (a ``trust-region").  See for example \citet{nocedal_book}. 

K-FAC uses the well-known Levenberg-Marquardt technique \citep{levenberg_marquardt} to automatically adapt the damping parameter $\lambdaParam$ so that the damping is loosened or tightened depending on how accurately $\quadModel(\paramVec + \paramUpdate)$ predicts the true decrease in the objective function after each step. This accuracy is measured by the so-called ``reduction ratio", which is given by
\begin{equation}
\rho = \frac{\objective(\paramVec) - \objective(\paramVec + \paramUpdate)}{\quadModel(\paramVec) - \quadModel(\paramVec + \paramUpdate)},
\end{equation}
and should be close to 1 when the quadratic approximation is reasonably accurate around the given value of $\paramVec$. The update rule for $\lambdaParam$ is as follows:
\begin{equation}
\lambdaParam \gets \left\{ \begin{array}{ll} \lambdaParam \cdot \lambdaParamDec & \textrm{if } \rhoStat > 3/4 \\ \lambdaParam & \textrm{if } 1/4 \leq \rhoStat \leq 3/4 \\ \lambdaParam \cdot \lambdaParamInc & \textrm{if } \rhoStat < 1/4 \end{array} \right.
\end{equation}
where $\lambdaParamInc$ and $\lambdaParamDec$ are constants such that $\lambdaParamDec < 1 < \lambdaParamInc$. 

To compensate for the inaccuracy of $\fisherMatApprox$, and encourage $\approxNatGrad$ to be smaller and more conservative, K-FAC similarly adds $\gammaParam \ident$ to $\fisherMatApprox$ before inverting it.  As discussed in Section \ref{sec:kfac}, this can be done approximately by adding multiples of $\ident$ to each of the Kronecker factors $\covActivationsL{\layerIdx}$ and $\covPreActivationGradientsL{\layerIdx}$ of $\fisherMatApproxL{\layerIdx}$ before inverting them.  Alternatively, an exact solution can be obtained by expanding out the eigendecomposition of each block $\fisherMatApproxL{\layerIdx}$ of $\fisherMatApprox$, and using the following identity:
\begin{align}
\left[ \fisherMatApproxL{\layerIdx} + \gammaParam \ident \right]^{-1} &= \left[ \left( \qMat_{\covActivations} \otimes \qMat_{\covPreActivationGradients} \right) \left( \dMat_{\covActivations} \otimes \dMat_{\covPreActivationGradients} \right) \left( \qMat_{\covActivations}^\transpose \otimes \qMat_{\covPreActivationGradients}^\transpose \right) + \gammaParam \ident \right]^{-1} \\
&= \left[ \left( \qMat_{\covActivations} \otimes \qMat_{\covPreActivationGradients} \right) \left( \dMat_{\covActivations} \otimes \dMat_{\covPreActivationGradients} + \gammaParam \ident \right) \left( \qMat_{\covActivations}^\transpose \otimes \qMat_{\covPreActivationGradients}^\transpose \right) \right]^{-1} \\
&= \left( \qMat_{\covActivations} \otimes \qMat_{\covPreActivationGradients} \right) \left( \dMat_{\covActivations} \otimes \dMat_{\covPreActivationGradients} + \gammaParam \ident \right)^{-1} \left( \qMat_{\covActivations}^\transpose \otimes \qMat_{\covPreActivationGradients}^\transpose \right), \label{eqn:kron_damping}
\end{align}
where $\covActivationsL{\layerIdx} = \qMat_{\covActivations} \dMat_{\covActivations} \qMat_{\covActivations}^\transpose$ and $\covPreActivationGradientsL{\layerIdx} = \qMat_{\covPreActivationGradients} \dMat_{\covPreActivationGradients} \qMat_{\covPreActivationGradients}^\transpose$ are the orthogonal eigendecompositions of $\covActivationsL{\layerIdx}$ and $\covPreActivationGradientsL{\layerIdx}$ (which are symmetric PSD). These manipulations are based on well-known properties of the Kronecker product which can be found in, e.g., \citet[sec.~6.3.3]{demmel}. 
Matrix-vector products $(\fisherMatApprox + \gammaParam \ident)^{-1} \nabla \objective$ can then be computed from the above identity using the following block-wise formulas:
\begin{align}
{\bf V}_1 &= \qMat_{\covPreActivationGradients}^\transpose (\nabla_{\weightsBiasesL{\layerIdx}} \objective) \qMat_{\covActivations} \\
{\bf V}_2 &= {\bf V}_1 / (\dVec_{\covPreActivationGradients} \dVec_{\covActivations}^\transpose + \gammaParam) \label{eqn:damping_elementwise_step} \\
(\fisherMatApproxL{\layerIdx} + \gammaParam \ident)^{-1} \kvec(\nabla_{\weightsBiasesL{\layerIdx}} \objective) &= \kvec \left( \qMat_{\covPreActivationGradients} {\bf V}_2 \qMat_{\covActivations}^\transpose \right), \label{eqn:damped_matrix_vector_products}
\end{align}
where $\dVec_{\covPreActivationGradients}$ and $\dVec_{\covActivations}$ are the diagonals of $\dMat_{\covPreActivationGradients}$ and $\dMat_{\covActivations}$ and the division and addition in Eqn.~\ref{eqn:damping_elementwise_step} are both elementwise.

One benefit of this damping strategy is that it automatically accounts for the curvature contributed by both the quadratic damping term $\frac{\lambdaParam}{2} \|\paramUpdate\|^2$ and the weight decay penalty $\frac{\weightDecayParam}{2} \|\paramVec\|^2$ if these are used. Heuristically, one could even set $\gammaParam = \sqrt{\lambdaParam + \weightDecayParam}$, which can sometimes perform well.  One should always choose $\gammaParam$ at least this large. However, it may sometimes be advantageous to choose $\gammaParam$ significantly larger, as $\fisherMatApprox$ might not be a good approximation to $\fisherMat$, and the damping may help reduce the impact of directions erroneously estimated to have low curvature. For consistency with \citet{kfac}, we adopt their method of automatically adapting $\gammaParam$. In particular, each time we adapt $\gammaParam$, we compute $\approxNatGrad$ for three different values $\gammaParamDec < \gamma < \gammaParamInc$. We choose whichever of the three values results in the lowest value of $\quadModel(\paramVec + \paramUpdate)$.

\subsection{Efficient implementation}
\label{app:implementation}

We based our implementation on the Toronto Deep Learning ConvNet (TDLCN) package \citep{tdlcn}, which is a Python wrapper around CUDA kernels. We needed to write a handful of additional kernels:
\begin{itemize}
\item a kernel for computing $\convKronActEmpL{\layerIdx}$ (Eqn.~\ref{eqn:conv_kron_emp})
\item kernels which performed forward mode automatic differentiation for the max-pooling and response normalization layers
\end{itemize}
Most of the other operations for KFC could be performed on the GPU using kernels provided by TDLCN. The only exception is computing the inverses $\{ \convKronActGammaLInv{\gammaParam}{\layerIdx} \}_{\layerIdx=0}^{\numLayers-1}$ and $\{ \convKronGradGammaLInv{\gammaParam}{\layerIdx} \}_{\layerIdx=1}^\numLayers$, which was done on the CPU. (The forward mode kernels are only used in update clipping; as mentioned above, one can likely eliminate this step in practice by initializing from a partially trained model.)

KFC introduces several sources of overhead per iteration compared with SGD:
\begin{itemize}
\item Updating the factors $\{ \convKronActL{\layerIdx} \}_{\layerIdx=0}^{\numLayers-1}$ and $\{ \convKronGradL{\layerIdx} \}_{\layerIdx=1}^\numLayers$
\item Computing the inverses $\{ \convKronActGammaLInv{\gammaParam}{\layerIdx} \}_{\layerIdx=0}^{\numLayers-1}$ and $\{ \convKronGradGammaLInv{\gammaParam}{\layerIdx} \}_{\layerIdx=1}^\numLayers$
\item Computing the approximate natural gradient $\approxNatGrad = [\fisherMatApproxGamma{\gammaParam}]^{-1} \nabla \objective$
\item Estimating $\updateNorm = \updateVec^\transpose \fisherMat \updateVec + \lambdaParam \updateVec^\transpose \updateVec$ (which is used for gradient clipping)
\end{itemize}
The overhead from the first two could be reduced by only periodically recomputing the factors and inverses, rather than doing so after every mini-batch. The cost of estimating $\updateVec^\transpose \fisherMat \updateVec$ can be reduced by using only a subset of the mini-batch. These shortcuts did not seem to hurt the per-epoch progress very much, as one can get away with using quite stale curvature information, and $\updateNorm$ is only used for clipping and therefore doesn't need to be very accurate. The cost of computing $\approxNatGrad$ is unavoidable, but because it doesn't grow with the size of the mini-batch, its per-epoch cost can be made smaller by using larger mini-batches. (As we discuss further in Section~\ref{sec:distributed}, KFC can work well with large mini-batches.) These shortcuts introduce several additional hyperparameters, but fortunately these are easy to tune: we simply chose them such that the per-epoch cost of KFC was less than twice that of SGD. This requires only running a profiler for a few epochs, rather than measuring overall optimization performance.

Observe that the inverses $\{ \convKronActGammaLInv{\gammaParam}{\layerIdx} \}_{\layerIdx=0}^{\numLayers-1}$ and $\{ \convKronGradGammaLInv{\gammaParam}{\layerIdx} \}_{\layerIdx=1}^\numLayers$  are computed on the CPU, while all of the other heavy computation is GPU-bound. In principle, since KFC works fine with stale curvature information, the inverses could be computed asychronously while the algorithm is running, thereby making their cost almost free. We did not exploit this in our experiments, however.

\section{Proofs of theorems}
\label{app:proofs}

\subsection{Proofs for Section~\ref{sec:kfc}}

\setcounter{theorem}{0}

\begin{lemma}
Under approximation {\bf IAD},
\begin{align}
\expect \left[ \gradWeightsIJS{\outputMapIdx}{\inputMapIdx}{\spatOffset} \gradWeightsIJS{\outputMapIdx^\prime}{\inputMapIdx^\prime}{\spatOffset^\prime} \right] &= \sum_{\spatIdx \in \spatIdxSet} \sum_{\spatIdx^\prime \in \spatIdxSet} \expect \left[ \activationsJS{\inputMapIdx}{\spatIdx+\spatOffset} \activationsJS{\inputMapIdx^\prime}{\spatIdx^\prime+\spatOffset^\prime} \right] \expect \left[ \gradPreActivationsIS{\outputMapIdx}{\spatIdx} \gradPreActivationsIS{\outputMapIdx^\prime}{\spatIdx^\prime} \right] \\
\expect \left[ \gradWeightsIJS{\outputMapIdx}{\inputMapIdx}{\spatOffset} \grad \biasI{\outputMapIdx^\prime} \right] &= \sum_{\spatIdx \in \spatIdxSet} \sum_{\spatIdx^\prime \in \spatIdxSet} \expect \left[ \activationsJS{\inputMapIdx}{\spatIdx+\spatOffset} \right] \expect \left[ \gradPreActivationsIS{\outputMapIdx}{\spatIdx} \gradPreActivationsIS{\outputMapIdx^\prime}{\spatIdx^\prime} \right] \\
\expect \left[ \grad \biasI{\outputMapIdx} \grad \biasI{\outputMapIdx^\prime} \right] &= \numLocs \, \expect \left[ \gradPreActivationsIS{\outputMapIdx}{\spatIdx} \gradPreActivationsIS{\outputMapIdx^\prime}{\spatIdx^\prime} \right]
\end{align}
\label{lem:approx_iad}
\end{lemma}

\begin{proof}
We prove the first equality; the rest are analogous.
\begin{align}
\expect[\gradWeightsIJS{\dimIdxOne}{\dimIdxTwo}{\spatOffset} \gradWeightsIJS{\dimIdxOne^\prime}{\dimIdxTwo^\prime}{\spatOffset^\prime}] 
&= \expect \left[ \left( \sum_{\spatIdx \in \spatIdxSet} \activationsJS{\inputMapIdx}{\spatIdx+\spatOffset} \gradPreActivationsIS{\outputMapIdx}{\spatIdx} \right) \left( \sum_{\spatIdx^\prime \in \spatIdxSet} \activationsJS{\inputMapIdx^\prime}{\spatIdx^\prime+\spatOffset^\prime} \gradPreActivationsIS{\outputMapIdx^\prime}{\spatIdx^\prime} \right) \right] \\
&= \expect \left[ \sum_{\spatIdx \in \spatIdxSet} \sum_{\spatIdx^\prime \in \spatIdxSet} \activationsJS{\inputMapIdx}{\spatIdx+\spatOffset} \gradPreActivationsIS{\outputMapIdx}{\spatIdx} \activationsJS{\inputMapIdx^\prime}{\spatIdx^\prime+\spatOffset^\prime} \gradPreActivationsIS{\outputMapIdx^\prime}{\spatIdx^\prime} \right] \\
&= \sum_{\spatIdx \in \spatIdxSet} \sum_{\spatIdx^\prime \in \spatIdxSet} \expect \left[ \activationsJS{\inputMapIdx}{\spatIdx+\spatOffset} \gradPreActivationsIS{\outputMapIdx}{\spatIdx} \activationsJS{\inputMapIdx^\prime}{\spatIdx^\prime+\spatOffset^\prime} \gradPreActivationsIS{\outputMapIdx^\prime}{\spatIdx^\prime} \right] \\
&= \sum_{\spatIdx \in \spatIdxSet} \sum_{\spatIdx^\prime \in \spatIdxSet} \expect \left[ \activationsJS{\inputMapIdx}{\spatIdx+\spatOffset} \activationsJS{\inputMapIdx^\prime}{\spatIdx^\prime+\spatOffset^\prime} \right] \expect \left[ \gradPreActivationsIS{\outputMapIdx}{\spatIdx} \gradPreActivationsIS{\outputMapIdx^\prime}{\spatIdx^\prime} \right] 
\end{align}
Assumption {\bf IAD} is used in the final line.
\end{proof}

\begin{theorem}
Combining approximations {\bf IAD}, {\bf SH}, and {\bf SUD} yields the following factorization:
\begin{align}
\expect \left[ \gradWeightsIJS{\outputMapIdx}{\inputMapIdx}{\spatOffset} \gradWeightsIJS{\outputMapIdx^\prime}{\inputMapIdx^\prime}{\spatOffset^\prime} \right] &= \boundaryFunction(\spatOffset, \spatOffset^\prime) \, \autoActivations(\inputMapIdx, \inputMapIdx^\prime, \spatOffset^\prime - \spatOffset) \, \autoGrad(\outputMapIdx, \outputMapIdx^\prime, 0), \nonumber \\
\expect \left[ \gradWeightsIJS{\outputMapIdx}{\inputMapIdx}{\spatOffset} \grad \biasI{\outputMapIdx^\prime} \right] &= \boundaryFunctionUni(\spatOffset) \, \meanActivations(\inputMapIdx) \, \autoGrad(\outputMapIdx, \outputMapIdx^\prime, 0) \nonumber \\
\expect \left[ \grad \biasI{\outputMapIdx} \grad \biasI{\outputMapIdx^\prime} \right] &= \numLocs \, \autoGrad(\outputMapIdx, \outputMapIdx^\prime, 0) 
\end{align}
where
\begin{align}
\boundaryFunctionUni(\spatOffset) &\triangleq \left( \numRows - |\spatOffsetRow| \right) \, \left( \numCols - |\spatOffsetCol| \right) \nonumber \\
\boundaryFunction(\spatOffset, \spatOffset^\prime) &\triangleq \left( \numRows - \max(\spatOffsetRow, \spatOffsetRow^\prime, 0) + \min(\spatOffsetRow, \spatOffsetRow^\prime, 0) \right) \cdot \left( \numCols - \max(\spatOffsetCol, \spatOffsetCol^\prime, 0) + \min(\spatOffsetCol, \spatOffsetCol^\prime, 0) \right)
\end{align}
\end{theorem}

\begin{proof}
\begin{align}
\expect[\gradWeightsIJS{\dimIdxOne}{\dimIdxTwo}{\spatOffset} \gradWeightsIJS{\dimIdxOne^\prime}{\dimIdxTwo^\prime}{\spatOffset^\prime}] 
&= \sum_{\spatIdx \in \spatIdxSet} \sum_{\spatIdx^\prime \in \spatIdxSet} \expect \left[ \activationsJS{\inputMapIdx}{\spatIdx+\spatOffset} \activationsJS{\inputMapIdx^\prime}{\spatIdx^\prime+\spatOffset^\prime} \right] \expect \left[ \gradPreActivationsIS{\outputMapIdx}{\spatIdx} \gradPreActivationsIS{\outputMapIdx^\prime}{\spatIdx^\prime} \right] \label{eqn:thm_one_iag} \\
&= \sum_{\spatIdx \in \spatIdxSet} \sum_{\spatIdx^\prime \in \spatIdxSet} \autoActivations(\inputMapIdx, \inputMapIdx^\prime, \spatIdx^\prime + \spatOffset^\prime - \spatIdx - \spatOffset) \, \doubleIndicator{\spatIdx + \spatOffset \in \spatIdxSet}{\spatIdx^\prime + \spatOffset^\prime \in \spatIdxSet} \, \autoGrad(\outputMapIdx, \outputMapIdx^\prime, \spatIdx^\prime - \spatIdx) \label{eqn:thm_one_sh} \\
&= \sum_{\spatIdx \in \spatIdxSet} \autoActivations(\inputMapIdx, \inputMapIdx^\prime, \spatOffset^\prime - \spatOffset) \, \doubleIndicator{\spatIdx + \spatOffset \in \spatIdxSet}{\spatIdx + \spatOffset^\prime \in \spatIdxSet} \, \autoGrad(\outputMapIdx, \outputMapIdx^\prime, 0) \label{eqn:thm_one_sud} \\
&= \left| \left\{ \spatIdx \in \spatIdxSet : \spatIdx + \spatOffset \in \spatIdxSet, \spatIdx + \spatOffset^\prime \in \spatIdxSet \right\} \right| \, \autoActivations(\inputMapIdx, \inputMapIdx^\prime, \spatOffset^\prime - \spatOffset) \, \autoGrad(\outputMapIdx, \outputMapIdx^\prime, 0) \\
&= \boundaryFunction(\spatOffset, \spatOffset^\prime) \, \autoActivations(\inputMapIdx, \inputMapIdx^\prime, \spatOffset^\prime - \spatOffset) \, \autoGrad(\outputMapIdx, \outputMapIdx^\prime, 0) \label{eqn:thm_one_final}
\end{align}
Lines \ref{eqn:thm_one_iag}, \ref{eqn:thm_one_sh}, and \ref{eqn:thm_one_sud} use Lemma \ref{lem:approx_iad} and assumptions {\bf SH}, and {\bf SUD}, respectively. In Line \ref{eqn:thm_one_sh}, the indicator function (denoted $\indicator$) arises because the activations are defined to be zero outside the set of spatial locations. The remaining formulas can be derived analogously.
\end{proof}

\begin{theorem}
Under assumption {\bf SH},
\begin{align}
\convKronActL{\layerIdx} &= \expect \left[ \expansion{\activationsMatL{\layerIdx}}_\homog^\transpose \expansion{\activationsMatL{\layerIdx}}_\homog \right] \\
\convKronGradL{\layerIdx} &= \frac{1}{\numLocs} \expect \left[ \grad \preActivationsMatL{\layerIdx}^\transpose \grad \preActivationsMatL{\layerIdx} \right]
\end{align}
\end{theorem}

\begin{proof}
In this proof, all activations and pre-activations are taken to be in layer $\layerIdx$. The expected entries are given by:
\begin{align}
\expect \left[ \expansion{\activationsMatL{\layerIdx}}_\homog^\transpose \expansion{\activationsMatL{\layerIdx}}_\homog \right]_{\inputMapIdx \numOffsets + \spatOffset, \, \inputMapIdx^\prime \numOffsets + \spatOffset} &= \expect \left[ \sum_{\spatIdx \in \spatIdxSet} \activationsJS{\inputMapIdx}{\spatIdx+\spatOffset} \activationsJS{\inputMapIdx^\prime}{\spatIdx+\spatOffset^\prime} \right] \\
&= \sum_{\spatIdx \in \spatIdxSet} \expect \left[ \activationsJS{\inputMapIdx}{\spatIdx+\spatOffset} \activationsJS{\inputMapIdx^\prime}{\spatIdx+\spatOffset^\prime} \right] \\
&= \sum_{\spatIdx \in \spatIdxSet} \autoActivations(\inputMapIdx, \inputMapIdx^\prime, \spatOffset^\prime - \spatOffset) \, \doubleIndicator{\spatIdx + \spatOffset \in \spatIdxSet}{\spatIdx + \spatOffset^\prime \in \spatIdxSet} \label{eqn:thm_two_sh} \\
&= \left| \left\{ \spatIdx \in \spatIdxSet : \spatIdx + \spatOffset \in \spatIdxSet, \spatIdx + \spatOffset^\prime \in \spatIdxSet \right\} \right| \, \autoActivations(\inputMapIdx, \inputMapIdx^\prime, \spatOffset^\prime - \spatOffset) \\
&= \boundaryFunction(\spatOffset, \spatOffset^\prime) \, \autoActivations(\inputMapIdx, \inputMapIdx^\prime, \spatOffset^\prime - \spatOffset) \\
&= [\convKronActL{\layerIdx}]_{\inputMapIdx \numOffsets + \spatOffset, \, \inputMapIdx^\prime \numOffsets + \spatOffset^\prime}
\end{align}
{\bf SH} is used in Line \ref{eqn:thm_two_sh}. Similarly,
\begin{align}
\expect \left[ \expansion{\activationsMatL{\layerIdx}}_\homog^\transpose \expansion{\activationsMatL{\layerIdx}}_\homog \right]_{0, \,  \inputMapIdx \numOffsets + \spatOffset} &= \expect \left[ \sum_{\spatIdx \in \spatIdxSet} \activationsJS{\inputMapIdx}{\spatIdx+\spatOffset} \right] \\
&= \boundaryFunctionUni(\spatOffset) \, \meanActivations(\inputMapIdx) \\
&= [\convKronActL{\layerIdx}]_{0, \, \inputMapIdx \numOffsets + \spatOffset} \\
\left[ \expansion{\activationsMatL{\layerIdx}}_\homog^\transpose \expansion{\activationsMatL{\layerIdx}}_\homog \right]_{0, \, 0} &= \numLocs \\
&= [\convKronActL{\layerIdx}]_{0, \, 0} \\
\expect \left[ \grad \preActivationsMatL{\layerIdx}^\transpose \grad \preActivationsMatL{\layerIdx} \right]_{\outputMapIdx, \outputMapIdx^\prime} &= \expect \left[ \sum_{\spatIdx \in \spatIdxSet} \grad \preActivationsIS{\outputMapIdx}{\spatIdx} \grad \preActivationsIS{\outputMapIdx^\prime}{\spatIdx} \right] \\
&= \numLocs \, \autoGrad(\outputMapIdx, \outputMapIdx^\prime, 0) \\
&= \numLocs \, \left[ \convKronGradL{\layerIdx} \right]_{\outputMapIdx,\, \outputMapIdx^\prime}
\end{align}
\end{proof}

\subsection{Proofs for Section~\ref{sec:theory}}

{\bf Preliminaries and notation.} In discussing invariances, it will be convenient to add homogeneous coordinates to various matrices:
\begin{align}
[\activationsMatL{\layerIdx}]_\homog &\triangleq \begin{pmatrix} \onesVec & \activationsMatL{\layerIdx} \end{pmatrix} \\
[\preActivationsMatL{\layerIdx}]_\homog &\triangleq \begin{pmatrix} \onesVec & \preActivationsMatL{\layerIdx} \end{pmatrix} \\
[\weightsBiasesL{\layerIdx}]_\homog &\triangleq \begin{pmatrix} 1 &  \\ \biasVecL{\layerIdx} & \weightsL{\layerIdx} \end{pmatrix}
\end{align}
We also define the activation function $\nonlinearity$ to ignore the homogeneous column, so that 
\begin{equation}
[\activationsMatL{\layerIdx}]_\homog = \nonlinearity([\preActivationsMatL{\layerIdx}]_\homog) = \nonlinearity(\expansion{\activationsMatL{\layerIdx-1}} [\weightsBiasesL{\layerIdx}]_\homog). \label{eqn:activation_function_homog}
\end{equation}

Using the homogeneous coordinate notation, we can write the effect of the affine transformations on the pre-activations and activations:
\begin{align}
[\preActivationsMatTransL{\layerIdx} \transMatPreL{\layerIdx} + \onesVec \transOffsetPreL{\layerIdx}^\transpose ]_\homog &= [\preActivationsMatTransL{\layerIdx}]_\homog [\transMatPreL{\layerIdx}]_\homog \nonumber \\
[\activationsMatL{\layerIdx} \transMatPostL{\layerIdx} + \onesVec \transOffsetPostL{\layerIdx}^\transpose]_\homog &= [\activationsMatL{\layerIdx}]_\homog [\transMatPostL{\layerIdx}]_\homog, \label{eqn:trans_activations_homog}
\end{align}
where
\begin{align}
[\transMatPreL{\layerIdx}]_\homog &\triangleq
\begin{pmatrix}
1 & \transOffsetPreL{\layerIdx}^\transpose \\
& \transMatPreL{\layerIdx}
\end{pmatrix}
\label{eqn:trans_mat_pre_homog} \\
[\transMatPostL{\layerIdx}]_\homog &\triangleq
\begin{pmatrix}
1 & \transOffsetPostL{\layerIdx}^\transpose \\
& \transMatPostL{\layerIdx}
\end{pmatrix}.
\label{eqn:trans_mat_post_homog}
\end{align}
The inverse transformations are represented as
\begin{align}
[\transMatPreL{\layerIdx}]_\homog^{-1} &\triangleq
\begin{pmatrix}
1 & -\transOffsetPreL{\layerIdx}^\transpose \transMatPreL{\layerIdx}^{-1} \\
& \transMatPreL{\layerIdx}^{-1}
\end{pmatrix}
\label{eqn:trans_mat_pre_homog_inv} \\
[\transMatPostL{\layerIdx}]_\homog^{-1} &\triangleq
\begin{pmatrix}
1 & -\transOffsetPostL{\layerIdx}^\transpose \transMatPostL{\layerIdx}^{-1} \\
& \transMatPostL{\layerIdx}^{-1}
\end{pmatrix}.
\label{eqn:trans_mat_post_homog_inv}
\end{align}

We can also determine the effect of the affine transformation on the \emph{expanded} activations:
\begin{equation}
\expansion{\activationsMatL{\layerIdx} \transMatPostL{\layerIdx} + \onesVec \transOffsetPostL{\layerIdx}^\transpose}_\homog = \expansion{\activationsMatL{\layerIdx}}_\homog \expansion{\transMatPostL{\layerIdx}}_\homog, \label{eqn:transformation_expanded}
\end{equation}
where
\begin{equation}
\expansion{\transMatPostL{\layerIdx}}_\homog \triangleq 
\begin{pmatrix}
1 & \transOffsetPostL{\layerIdx}^\transpose \otimes \onesVec^\transpose \\
& \transMatPostL{\layerIdx} \otimes \ident
\end{pmatrix},
\label{eqn:expand_trans_mat_post}
\end{equation}
with inverse
\begin{equation}
\expansion{\transMatPostL{\layerIdx}}_\homog^{-1} = \begin{pmatrix}
1 & -\transOffsetPostL{\layerIdx}^\transpose \transMatPostL{\layerIdx}^{-1} \otimes \onesVec^\transpose \\
& \transMatPostL{\layerIdx}^{-1} \otimes \ident
\end{pmatrix}.
\end{equation}
Note that $\expansion{\transMatPostL{\layerIdx}}_\homog$ is simply a suggestive notation, rather than an application of the expansion operator $\expansion{\cdot}$.

\begin{lemma}
Let $\network$, $\paramVec$, $\{\nonlinearityL{\layerIdx}\}_{\layerIdx=0}^\numLayers$, and $\{\nonlinearityTransL{\layerIdx}\}_{\layerIdx=0}^\numLayers$ be given as in Theorem \ref{thm:invariance}. The network $\networkTrans$ with activations functions $\{\nonlinearityTransL{\layerIdx}\}_{\layerIdx=0}^\numLayers$ and parameters defined by 
\begin{equation}
[\weightsBiasesTransL{\layerIdx}]_\homog \triangleq [\transMatPreL{\layerIdx}]_\homog^{-\transpose} [\weightsBiasesL{\layerIdx}]_\homog \expansion{\transMatPostL{\layerIdx-1}}_\homog^{-\transpose}, \label{eqn:weights_biases_homog_trans}
\end{equation}
compute the same function as $\network$. 
\label{lem:param_transformation}
\end{lemma}

\noindent {\bf Remark.} The definition of $\nonlinearityTransL{\layerIdx}$ (Eqn.~\ref{eqn:affine_transformation}) can be written in homogeneous coordinates as
\begin{equation}
[\activationsMatTransL{\layerIdx}]_\homog = \nonlinearityTransL{\layerIdx}([\preActivationsMatTransL{\layerIdx}]_\homog) = \nonlinearityL{\layerIdx}([\preActivationsMatTransL{\layerIdx}]_\homog [\transMatPreL{\layerIdx}]_\homog) [\transMatPostL{\layerIdx}]_\homog.
\end{equation}
Eqn.~\ref{eqn:weights_biases_homog_trans} can be expressed equivalently without homogeneous coordinates as
\begin{equation}
\weightsBiasesTransL{\layerIdx} \triangleq \transMatPreL{\layerIdx}^{-\transpose} \left( \weightsBiasesL{\layerIdx} - \transOffsetPreL{\layerIdx} \indicatorVec^\transpose \right) \expansion{\transMatPostL{\layerIdx-1}}_\homog^{-\transpose}, \label{eqn:weights_biases_trans}
\end{equation}
where $\indicatorVec = (1\ 0\ \cdots\ 0)^\transpose$.

\begin{proof}
We will show inductively the following relationship between the activations in each layer of the two networks:
\begin{align}
[\activationsMatTransL{\layerIdx}]_\homog = [\activationsMatL{\layerIdx}]_\homog [\transMatPostL{\layerIdx}]_\homog. \label{eqn:inductive_hypothesis}
\end{align}
(By our assumption that the top layer inputs are not transformed, i.e. $[\transMatPostL{\numLayers}]_\homog = \ident$, this would imply that $[\activationsMatTransL{\numLayers}]_\homog = [\activationsMatL{\numLayers}]_\homog$, and hence that the networks compute the same function.) For the first layer, Eqn.~\ref{eqn:inductive_hypothesis} is true by definition. For the inductive step, assume Eqn.~\ref{eqn:inductive_hypothesis} holds for layer $\layerIdx - 1$. From Eqn~\ref{eqn:transformation_expanded}, this is equivalent to
\begin{equation}
\expansion{\activationsMatTransL{\layerIdx-1}}_\homog = \expansion{\activationsMatL{\layerIdx-1}}_\homog \expansion{\transMatPostL{\layerIdx-1}}_\homog. \label{eqn:inductive_hypothesis_expanded}
\end{equation}
We then derive the activations in the following layer:
\begin{align}
[\activationsMatTransL{\layerIdx}]_\homog &= \nonlinearityTransL{\layerIdx} \left( [\preActivationsMatTransL{\layerIdx}]_\homog \right) \\
&= \nonlinearityL{\layerIdx} \left( [\preActivationsMatTransL{\layerIdx}]_\homog \, [\transMatPreL{\layerIdx}]_\homog \right) \, [\transMatPostL{\layerIdx}]_\homog \label{eqn:step_def_phi_trans} \\
&= \nonlinearityL{\layerIdx} \left( \expansion{\activationsMatTransL{\layerIdx-1}}_\homog \, [\weightsBiasesTransL{\layerIdx}]_\homog^\transpose \, [\transMatPreL{\layerIdx}]_\homog \right) \, [\transMatPostL{\layerIdx}]_\homog \label{eqn:step_def_pre_activations} \\
&= \nonlinearityL{\layerIdx} \left( \expansion{\activationsMatL{\layerIdx-1}}_\homog \, \expansion{\transMatPostL{\layerIdx-1}}_\homog \, [\weightsBiasesTransL{\layerIdx}]_\homog^\transpose \, [\transMatPreL{\layerIdx}]_\homog \right) \, [\transMatPostL{\layerIdx}]_\homog \label{eqn:step_inductive_hypothesis} \\
&= \nonlinearityL{\layerIdx} \left( \expansion{\activationsMatL{\layerIdx-1}}_\homog \, \expansion{\transMatPostL{\layerIdx-1}}_\homog \, \expansion{\transMatPostL{\layerIdx-1}}_\homog^{-1} \, [\weightsBiasesL{\layerIdx}]_\homog^\transpose \,[\transMatPreL{\layerIdx}]_\homog^{-1} \, [\transMatPreL{\layerIdx}]_\homog \right) \, [\transMatPostL{\layerIdx}]_\homog \label{eqn:step_def_weights_trans} \\
&= \nonlinearityL{\layerIdx} \left( \expansion{\activationsMatL{\layerIdx-1}}_\homog \, [\weightsBiasesL{\layerIdx}]_\homog^\transpose \right) \, [\transMatPostL{\layerIdx}]_\homog \\
&= [\activationsMatL{\layerIdx}]_\homog \, [\transMatPostL{\layerIdx}]_\homog \label{eqn:step_layer_computation}
\end{align}
Lines \ref{eqn:step_def_pre_activations} and \ref{eqn:step_layer_computation} are from Eqn.~\ref{eqn:activation_function_homog}. This proves the inductive hypothesis for layer $\layerIdx$, so we have shown that both networks compute the same function.
\end{proof}

\begin{lemma}
Suppose the parameters are transformed according to Lemma \ref{lem:param_transformation}, and the parameters are updated according to 
\begin{equation}
[\weightsBiasesTransL{\layerIdx}]^{(\iterCount+1)} \gets [\weightsBiasesTransL{\layerIdx}]^{(\iterCount)} - \learningRate \leftPreconditionerTransL{\layerIdx} ( \nabla_{\weightsBiasesTransL{\layerIdx}} \objective ) \rightPreconditionerTransL{\layerIdx},
\end{equation}
for matrices $\leftPreconditionerL{\layerIdx}$ and $\rightPreconditionerL{\layerIdx}$. This is equivalent to applying the following update to the original network:
\begin{equation}
[\weightsBiasesL{\layerIdx}]^{(\iterCount+1)} \gets [\weightsBiasesL{\layerIdx}]^{(\iterCount+1)} - \learningRate  \leftPreconditionerL{\layerIdx} ( \nabla_{\weightsBiasesL{\layerIdx}} \objective ) \rightPreconditionerL{\layerIdx},
\end{equation}
with
\begin{align}
\leftPreconditionerL{\layerIdx} &= \transMatPreL{\layerIdx}^\transpose \leftPreconditionerTransL{\layerIdx} \transMatPreL{\layerIdx} \\
\rightPreconditionerL{\layerIdx} &= \expansion{\transMatPostL{\layerIdx-1}}_\homog \rightPreconditionerTransL{\layerIdx} \expansion{\transMatPostL{\layerIdx-1}}_\homog^\transpose.
\end{align}

\label{lem:transformed_update}
\end{lemma}

\begin{proof}
This is a special case of Lemma 5 from \citet{kfac}.
\end{proof}

\begin{theorem}
Let $\network$ be a network with parameter vector $\paramVec$ and activation functions $\{\nonlinearityL{\layerIdx}\}_{\layerIdx=0}^\numLayers$. Given activation functions $\{\nonlinearityTransL{\layerIdx}\}_{\layerIdx=0}^\numLayers$ defined as in Eqn.~\ref{eqn:affine_transformation}, there exists a parameter vector $\paramVecTrans$ such that a network $\networkTrans$ with parameters $\paramVecTrans$ and activation functions $\{\nonlinearityTransL{\layerIdx}\}_{\layerIdx=0}^\numLayers$ computes the same function as $\network$. The KFC updates on $\network$ and $\networkTrans$ are equivalent, in that the resulting networks compute the same function.
\end{theorem}

\begin{proof}
Lemma \ref{lem:param_transformation} gives the desired $\paramVecTrans$. We now prove equivalence of the KFC updates. The Kronecker factors for $\networkTrans$ are given by:
\begin{align}
\convKronActTransL{\layerIdx} &= \expect \left[ \expansion{\activationsMatTransL{\layerIdx}}_\homog^\transpose \expansion{\activationsMatTransL{\layerIdx}}_\homog \right] \\
&= \expect \left[ \expansion{\transMatPostL{\layerIdx}}_\homog^\transpose \expansion{\activationsMatL{\layerIdx}}_\homog^\transpose \expansion{\activationsMatL{\layerIdx}}_\homog \expansion{\transMatPostL{\layerIdx}}_\homog \right] \\
&= \expansion{\transMatPostL{\layerIdx}}_\homog^\transpose \expect \left[ \expansion{\activationsMatL{\layerIdx}}_\homog^\transpose \expansion{\activationsMatL{\layerIdx}}_\homog \right] \expansion{\transMatPostL{\layerIdx}}_\homog \\
&= \expansion{\transMatPostL{\layerIdx}}_\homog^\transpose \convKronActL{\layerIdx} \expansion{\transMatPostL{\layerIdx}}_\homog \\
\convKronGradTransL{\layerIdx} &= \frac{1}{\numLocs} \expect \left[ (\grad \preActivationsMatTransL{\layerIdx})^\transpose \grad \preActivationsMatTransL{\layerIdx} \right] \\
&= \frac{1}{\numLocs} \expect \left[ \transMatPreL{\layerIdx} (\grad \preActivationsMatTransL{\layerIdx})^\transpose \grad \preActivationsMatTransL{\layerIdx} \transMatPreL{\layerIdx}^\transpose \right] \\
&= \frac{1}{\numLocs} \transMatPreL{\layerIdx} \expect \left[ (\grad \preActivationsMatTransL{\layerIdx})^\transpose \grad \preActivationsMatTransL{\layerIdx} \right] \transMatPreL{\layerIdx}^\transpose \\
&= \transMatPreL{\layerIdx} \convKronGradL{\layerIdx} \transMatPreL{\layerIdx}^\transpose
\end{align}

The approximate natural gradient update, ignoring momentum, clipping, and damping, is given by $\paramVec^{(\iterCount+1)} \gets \paramVec^{(\iterCount)} - \learningRate \fisherMatApprox^{-1} \nabla_\paramVec \objective.$ For each layer of $\networkTrans$,
\begin{equation}
[\weightsBiasesTransL{\layerIdx}]^{(\iterCount+1)} \gets [\weightsBiasesTransL{\layerIdx}]^{(\iterCount)} - \learningRate ( \convKronGradTransL{\layerIdx} )^{-1} ( \nabla_{\weightsBiasesTransL{\layerIdx}} \objective ) ( \convKronActTransL{\layerIdx-1} )^{-1} \\
\end{equation}
We apply Lemma~\ref{lem:transformed_update} with $\leftPreconditionerTransL{\layerIdx} = ( \convKronGradTransL{\layerIdx} )^{-1}$ and $\rightPreconditionerTransL{\layerIdx} = ( \convKronActTransL{\layerIdx-1} )^{-1}$. This gives us 
\begin{align}
\leftPreconditionerL{\layerIdx} &= \transMatPreL{\layerIdx}^\transpose ( \convKronGradTransL{\layerIdx} )^{-1} \transMatPreL{\layerIdx} \\
&= \convKronGradL{\layerIdx}^{-1} \\
\rightPreconditionerL{\layerIdx} &= \expansion{\transMatPostL{\layerIdx-1}}_\homog ( \convKronActTransL{\layerIdx-1} )^{-1} \expansion{\transMatPostL{\layerIdx-1}}_\homog^\transpose \\
&= \convKronActL{\layerIdx-1}^{-1},
\end{align}
with the corresponding update
\begin{equation}
[\weightsBiasesL{\layerIdx}]^{(\iterCount+1)} \gets [\weightsBiasesL{\layerIdx}]^{(\iterCount)} - \learningRate \convKronGradL{\layerIdx}^{-1} (\nabla_{\weightsBiasesL{\layerIdx}} \objective) \convKronActL{\layerIdx-1}^{-1}.
\end{equation}
But this is the same as the KFC update for the original network. Therefore, the two updates are equivalent, in that the resulting networks compute the same function.
\end{proof}

\begin{theorem}
Combining approximations {\bf IAD}, {\bf SH}, {\bf SUA}, and {\bf WD} results in the following approximation to the entries of the Fisher matrix:
\begin{align}
\expect \left[ \gradWeightsIJS{\outputMapIdx}{\inputMapIdx}{\spatOffset} \gradWeightsIJS{\outputMapIdx^\prime}{\inputMapIdx^\prime}{\spatOffset^\prime} \right] &= \boundaryFunction(\spatOffset, \spatOffset^\prime) \, \autoActivationsNNN(\inputMapIdx, \inputMapIdx^\prime, \spatOffset^\prime - \spatOffset) \, \kronDelta{\outputMapIdx}{\outputMapIdx^\prime} \\
\expect \left[ \gradWeightsIJS{\outputMapIdx}{\inputMapIdx}{\spatOffset} \grad \biasI{\outputMapIdx^\prime} \right] &= \boundaryFunction(\spatOffset) \, \meanActivations(\inputMapIdx) \, \kronDelta{\outputMapIdx}{\outputMapIdx^\prime} \\
\expect \left[ \grad \biasI{\outputMapIdx} \grad \biasI{\outputMapIdx^\prime} \right] &= \numLocs\, \kronDelta{\outputMapIdx}{\outputMapIdx^\prime}
\end{align}
where $\indicator$ is the indicator function and $\autoActivationsNNN(\inputMapIdx, \inputMapIdx^\prime, \spatOffset) = \autoCovActivations(\inputMapIdx, \inputMapIdx^\prime) \kronDelta{\spatOffset}{0} + \meanActivations(\inputMapIdx) \meanActivations(\inputMapIdx^\prime)$ is the uncentered autocovariance function. ($\boundaryFunction$ is defined in Theorem \ref{thm:kfc_sud}.) If the $\boundaryFunction$ and $\numLocs$ terms are dropped, the resulting approximate natural gradient descent update rule is equivalent to idealized PRONG, up to rescaling.
\end{theorem}

\begin{proof}
We first compute the second moments of the activations and derivatives, under assumptions {\bf SH}, {\bf SUA}, and {\bf WD}:
\begin{align}
\expect \left[ \activationsJS{\inputMapIdx}{\spatIdx} \activationsJS{\inputMapIdx^\prime}{\spatIdx^\prime} \right] 
&= \Cov(\activationsJS{\inputMapIdx}{\spatIdx}, \activationsJS{\inputMapIdx^\prime}{\spatIdx^\prime}) + \expect[\activationsJS{\inputMapIdx}{\spatIdx}] \expect[ \activationsJS{\inputMapIdx^\prime}{\spatIdx^\prime}] \\
&= \autoCovActivations(\inputMapIdx, \inputMapIdx^\prime) \kronDelta{\spatOffset}{0} + \meanActivations(\inputMapIdx) \meanActivations(\inputMapIdx^\prime) \\
&\triangleq \autoActivationsNNN(\inputMapIdx, \inputMapIdx^\prime, \spatOffset) \\
\expect \left[ \gradPreActivationsIS{\outputMapIdx}{\spatIdx} \gradPreActivationsIS{\outputMapIdx^\prime}{\spatIdx^\prime} \right] 
&= \kronDelta{\outputMapIdx}{\outputMapIdx^\prime} \kronDelta{\spatOffset}{\spatOffset^\prime}.
\end{align}
for any $\spatIdx, \spatIdx^\prime \in \spatIdxSet$. We now compute
\begin{align}
\expect \left[ \gradWeightsIJS{\outputMapIdx}{\inputMapIdx}{\spatOffset} \gradWeightsIJS{\outputMapIdx}{\inputMapIdx}{\spatOffset} \right] 
&= \sum_{\spatIdx \in \spatIdxSet} \sum_{\spatIdx^\prime \in \spatIdxSet} \expect \left[ \activationsJS{\inputMapIdx}{\spatIdx+\spatOffset} \activationsJS{\inputMapIdx^\prime}{\spatIdx^\prime+\spatOffset^\prime} \right] \expect \left[ \gradPreActivationsIS{\outputMapIdx}{\spatIdx} \gradPreActivationsIS{\outputMapIdx^\prime}{\spatIdx^\prime} \right] \label{eqn:thm_prong_iag} \\
&= \sum_{\spatIdx \in \spatIdxSet} \sum_{\spatIdx^\prime \in \spatIdxSet} \autoActivationsNNN(\inputMapIdx, \inputMapIdx^\prime, \spatIdx^\prime + \spatOffset^\prime - \spatIdx - \spatOffset) \, \doubleIndicator{\spatIdx + \spatOffset \in \spatIdxSet}{\spatIdx^\prime + \spatOffset^\prime \in \spatIdxSet} \kronDelta{\outputMapIdx}{\outputMapIdx^\prime} \kronDelta{\spatIdx}{\spatIdx^\prime} \\
&= \sum_{\spatIdx \in \spatIdxSet} \autoActivationsNNN(\inputMapIdx, \inputMapIdx^\prime, \spatOffset^\prime - \spatOffset) \, \doubleIndicator{\spatIdx + \spatOffset \in \spatIdxSet}{\spatIdx + \spatOffset^\prime \in \spatIdxSet} \kronDelta{\outputMapIdx}{\outputMapIdx^\prime} \\
&= \left| \left\{ \spatIdx \in \spatIdxSet : \spatIdx + \spatOffset \in \spatIdxSet, \spatIdx + \spatOffset^\prime \in \spatIdxSet \right\} \right| \autoActivationsNNN(\inputMapIdx, \inputMapIdx^\prime, \spatOffset^\prime - \spatOffset) \, \kronDelta{\outputMapIdx}{\outputMapIdx^\prime} \\
&= \boundaryFunction(\spatOffset, \spatOffset^\prime) \, \autoActivationsNNN(\inputMapIdx, \inputMapIdx^\prime, \spatOffset^\prime - \spatOffset) \, \kronDelta{\outputMapIdx}{\outputMapIdx^\prime} 
\end{align}
Line \ref{eqn:thm_prong_iag} is from Lemma \ref{lem:approx_iad}. The other formulas are derived analogously.

This can be written in matrix form as
\begin{align}
\fisherMatApprox &= \convKronActNNN \otimes \ident \label{eqn:prong_fisher_block} \\
\convKronActNNN &\triangleq \begin{pmatrix}
1 & \meanVecNNN^\transpose \otimes \onesVec^\transpose \\
\meanVecNNN \otimes \onesVec & \covMatNNN \otimes \ident + \meanVecNNN \meanVecNNN^\transpose \otimes \onesVec \onesVec^\transpose
\end{pmatrix}
\end{align}
It is convenient to compute block Cholesky decompositions:
\begin{align}
\convKronActNNN &= \begin{pmatrix}
1 & \\
\meanVecNNN \otimes \onesVec & \covMatNNNSqrt \otimes \ident 
\end{pmatrix} \begin{pmatrix}
1 & \meanVecNNN^\transpose \otimes \onesVec^\transpose \\
& \covMatNNNSqrt^\transpose \otimes \ident
\end{pmatrix} \\
&\triangleq \cholFactorNNN \cholFactorNNN^\transpose \\
\convKronActNNN^{-1} &= \cholFactorNNN^{-\transpose} \cholFactorNNN^{-1} \\
&= \begin{pmatrix}
1 & -\meanVecNNN^\transpose \covMatNNNSqrt^{-\transpose} \otimes \onesVec^\transpose \\
& \covMatNNNSqrt^{-\transpose} \otimes \ident
\end{pmatrix} \begin{pmatrix}
1 & \\
-\covMatNNNSqrt^{-1} \meanVecNNN \otimes \onesVec & \covMatNNNSqrt^{-1} \otimes \ident 
\end{pmatrix}, \label{eqn:chol_factor_conv_kron_act_nnn}
\end{align}
where $\covMatNNNSqrt$ is some square root matrix, i.e.~$\covMatNNNSqrt \covMatNNNSqrt^\transpose = \covMatNNN$ (not necessarily lower triangular). 

Now consider PRONG. In the original algorithm, the network is periodically reparameterized such that the activations are white. In our idealized version of the algorithm, we assume this is done after every update. For convenience, we assume that the network is converted to the white parameterizaton immediately before computing the SGD update, and then converted back to its original parameterization immediately afterward. In other words, we apply an affine transformation (Eqn.~\ref{eqn:affine_transformation}) which whitens the activations:
\begin{align}
\activationsMatTransL{\layerIdx} = \nonlinearityTransL{\layerIdx}(\preActivationsMatTransL{\layerIdx}) &= \left( \nonlinearityL{\layerIdx}(\preActivationsMatTransL{\layerIdx}) - \onesVec \meanVecNNN^\transpose \right) \covMatNNNSqrt^{-1} \\
&= \nonlinearityL{\layerIdx}(\preActivationsMatTransL{\layerIdx}) \covMatNNNSqrt^{-1} - \onesVec \meanVecNNN^\transpose \covMatNNNSqrt^{-1},
\end{align}
where $\covMatNNNSqrt$ is a square root matrix of $\covMatNNN$, as defined above. This is an instance of Eqn.~\ref{eqn:affine_transformation} with $\transMatPreL{\layerIdx} = \ident$, $\transOffsetPreL{\layerIdx} = \zeroMat$, $\transMatPostL{\layerIdx} = \covMatNNNSqrt^{-1}$, and $\transOffsetPostL{\layerIdx} = -\covMatNNNSqrt^{-1} \meanVecNNN$.
The transformed weights which compute the same function as the original network according to Lemma~\ref{lem:param_transformation} are $\weightsBiasesTransL{\layerIdx} = \weightsBiasesL{\layerIdx} \, \expansion{\covMatNNNSqrt^{-1}}_\homog^{-\transpose}$, where 
\begin{equation}
\expansion{\covMatNNNSqrt^{-1}}_\homog \triangleq \begin{pmatrix}
1 & -\meanVecNNN^\transpose \covMatNNNSqrt^{-\transpose} \otimes \onesVec^\transpose \\
& \covMatNNNSqrt^{-1} \otimes \ident
\end{pmatrix},
\end{equation}
is defined according to Eqn.~\ref{eqn:expand_trans_mat_post}. But observe that $\expansion{\covMatNNNSqrt^{-1}}_\homog = \cholFactorNNN^{-\transpose}$, where $\cholFactorNNN$ is the Cholesky factor of $\convKronActNNN$ (Eqn.~\ref{eqn:chol_factor_conv_kron_act_nnn}). Therefore, we have
\begin{equation}
\weightsBiasesTransL{\layerIdx} = \weightsBiasesL{\layerIdx} \, \cholFactorNNN.
\end{equation}

We apply Lemma \ref{lem:transformed_update} with $\leftPreconditionerTransL{\layerIdx} = \ident$ and $\rightPreconditionerTransL{\layerIdx} = \ident$. This gives us the update in the original coordinate system:
\begin{align}
\weightsBiasesL{\layerIdx}^{(\iterCount+1)} &\gets \weightsBiasesL{\layerIdx}^{(\iterCount)} - \learningRate (\nabla_{\weightsBiasesL{\layerIdx}} \objective) \, \cholFactorNNN^{-\transpose} \cholFactorNNN^{-1} \\
&= \weightsBiasesL{\layerIdx}^{(\iterCount)} - \learningRate (\nabla_{\weightsBiasesL{\layerIdx}} \objective) \, \convKronActNNN^{-1}.
\end{align}
This is equivalent to the approximate natural gradient update where the Fisher block is approximated as $\convKronActNNN \otimes \ident$. This is the same approximate Fisher block we derived given the assumptions of the theorem (Eqn.~\ref{eqn:prong_fisher_block}).
\end{proof}

\end{document}